%% file: acl_latex.tex
\pdfoutput=1

\documentclass[11pt]{article}

\usepackage[final]{acl}

\usepackage{times}
\usepackage{latexsym}
\usepackage{inconsolata}
\usepackage{booktabs}
\usepackage{multirow}
\usepackage{graphicx}
\usepackage{amssymb}
\usepackage{mathrsfs}
\usepackage{amsmath}
\usepackage{algorithm}
\usepackage{algorithmic}
\usepackage{bbding}
\usepackage{colortbl}
\usepackage{tcolorbox}
\usepackage{amsthm}
\usepackage{url}
\usepackage{wrapfig}
\usepackage{enumitem}
\newtheorem{theorem}{Theorem}
\newtheorem{assumption}{Assumption}
\usepackage[T1]{fontenc}

\usepackage[utf8]{inputenc}

\usepackage{microtype}

\usepackage{inconsolata}

\usepackage{graphicx}

\title{Selective Preference Optimization via Token-Level Reward Function Estimation}

\author{Kailai Yang\textsuperscript{1}\textsuperscript{*}\quad Zhiwei Liu\textsuperscript{1}\textsuperscript{\dag}
\quad 
\textbf{Qianqian Xie}\textsuperscript{\textbf{2},\textbf{3}}\quad \textbf{Jimin Huang}\textsuperscript{\textbf{4}} \\
\textbf{Erxue Min}\textsuperscript{\textbf{5}} \quad
\textbf{Sophia Ananiadou}\textsuperscript{\textbf{1},\textbf{6}}\\
    \textsuperscript{1} National Centre for Text Mining, The University of Manchester\\
    \textsuperscript{2} School of Artificial Intelligence, Wuhan University\\ 
    \textsuperscript{3} Center for Language and Information Research, Wuhan University\\
    \textsuperscript{4} The Fin AI \quad
    \textsuperscript{5} Baidu Inc.\quad
    \textsuperscript{6} Archimedes Research\\
\texttt{\{kailai.yang,zhiwei.liu,sophia.ananiadou\}@manchester.ac.uk}\\
\texttt{\{xqq.sincere,erxue.min\}@gmail.com;jimin@chancefocus.com}\\
}

\begin{document}

\maketitle
\begin{abstract}
Recent advancements in LLM alignment leverage token-level supervisions to perform fine-grained preference optimization. However, existing token-level alignment methods either optimize on all available tokens, which can be noisy and inefficient, or perform selective training with complex and expensive key token selection strategies.
In this work, we propose Selective Preference Optimization (SePO), a novel selective alignment strategy that centers on efficient key token selection without requiring strong, fine-grained supervision signals. We prove the feasibility of Direct Preference Optimization (DPO) as token-level reward function estimators, which applies to any existing alignment datasets and enables cost-efficient token selection with small-scale model sizes and training data. We then train an oracle model with DPO on the target data and utilize the estimated reward function to score all tokens within the target dataset, where only the key tokens are selected to supervise the target policy model with a contrastive objective function. Extensive experiments on three public evaluation benchmarks show that SePO significantly outperforms competitive baseline methods by only optimizing on 30\% key tokens with up to 60\% reduction in  GPU training hours. We also explore SePO as a new paradigm for weak-to-strong generalization, showing
that weak oracle models effectively supervise strong policy models with up to 16.8$\times$ more parameters. SePO also selects useful supervision signals from out-of-distribution data, alleviating the over-optimization problem. The project is \href{https://github.com/SteveKGYang/SePO}{open-sourced here}.
\end{abstract}

\renewcommand{\thefootnote}{\fnsymbol{footnote}}
\footnotetext[1]{Work done in collaboration with Baidu Search Science.}
\footnotetext[2]{Corresponding author.}
\renewcommand{\thefootnote}{\arabic{footnote}}

\section{Introduction}

The recent development of large language models (LLMs) has focused on aligning model outputs with human preferences~\citep{ji2023ai}. During alignment, LLMs are directly optimized on pairwise data and response-level supervision,
where popular methods such as reinforcement learning from human feedback (RLHF)~\citep{ouyang2022training,stiennon2020learning} and direct alignment algorithms~\citep{rafailov2024direct,yuan2023rrhf,meng2024simpo} only introduce supervision signals at the end of each response. As deriving preference optimization as bandit problems can lead to sub-optimal solutions and unstable training processes~\citep{zhong2024dpo,zeng2024token}, many works propose to model LLM decoding as token-level Markov Decision Processes (MDP) and introduce step-wise supervision signals that quantify the value of each action, successfully applied in tasks such as instruction following~\citep{zhong2024dpo,yoon2024tlcr} and mathematical reasoning~\citep{xie2024monte,chen2024step,lai2024step}. 

\begin{figure}[htpb]
\centering
\includegraphics[width=7cm,height=4.67cm]{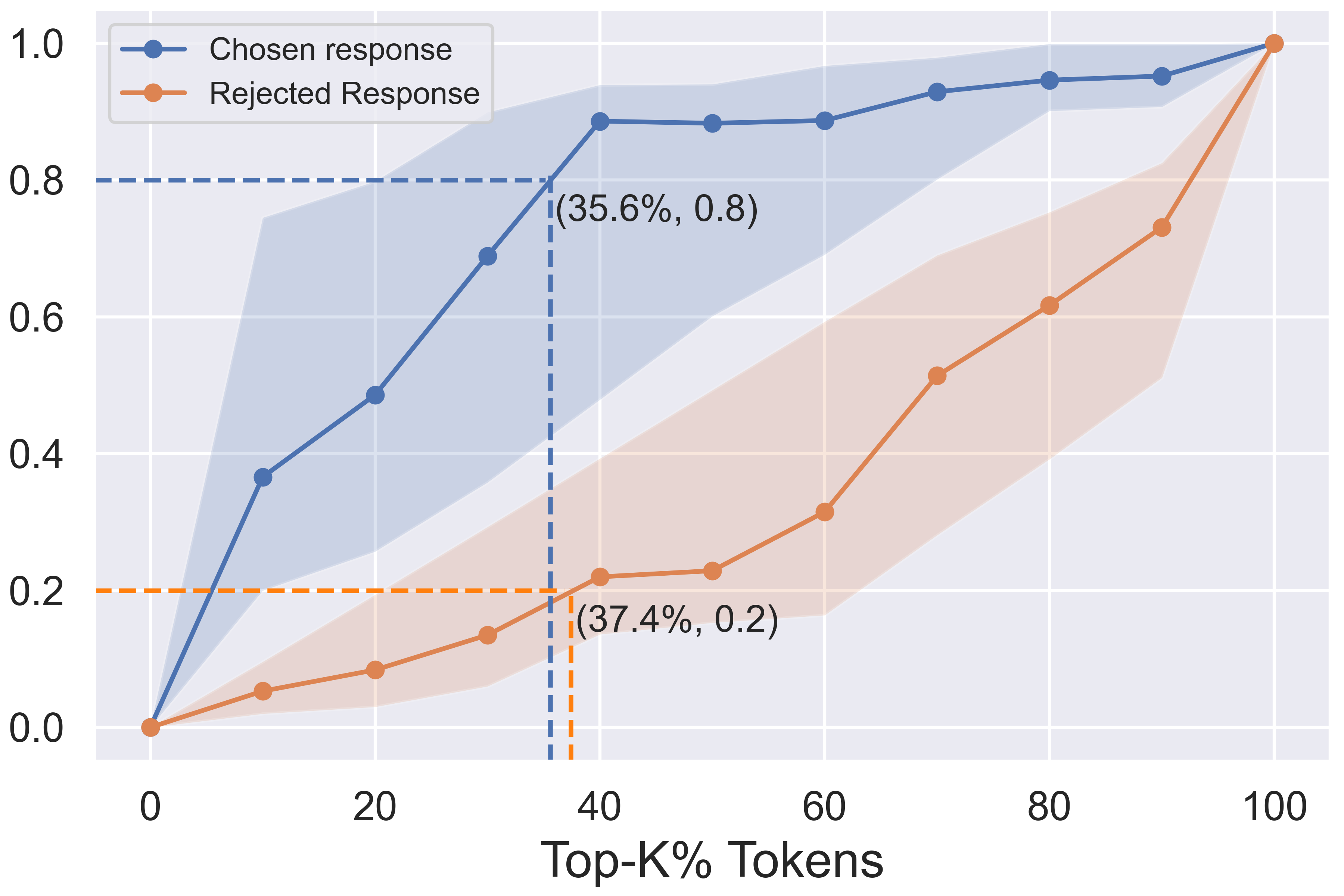}
\caption{Token-level reward accumulations. As tokens with high rewards are considered key tokens for chosen responses, their Top-K\% tokens are accumulated in descending order with the highest rewards. In contrast, rewards are accumulated in ascending orders for rejected responses. More details in Appendix \ref{appn:reward_accumulate}.}
\label{fig:reward_distribution}
\end{figure}

Though achieving outstanding performance, most of these methods are optimized on all available tokens from the training dataset. To validate the effectiveness of this setup, in Figure \ref{fig:reward_distribution}, we present the token-level reward accumulations for 1,000 samples from an instruction following dataset~\citep{cui2023ultrafeedback} 
where the token-level rewards are assigned by GPT-4~\citep{achiam2023gpt}. According to the results, a limited number of tokens with extreme reward values (key tokens) dominate the total rewards. The Top-35.6\% tokens occupy the highest 80\% rewards for chosen responses, while the lowest 37.4\% tokens only occupy 20\% rewards for rejected responses. 
These observations prove that not all tokens are equally effective in preference alignment, and optimizing on all available tokens can be noisy and inefficient~\citep{lin2024rho,chen2024low}. Some works explored only optimizing on selected response fragments, but their selection strategies are complex and expensive, utilizing Monte-Carlo Tree Search~\citep{xie2024monte,chen2024step} or annotations from human/capable LLMs~\citep{lai2024step,yoon2024tlcr}. The above limitations underscore the prospect of selective training and more efficient token selection strategies in improving preference optimization algorithms.

Based on these intuitions, we propose Selective Preference Optimization (SePO), a novel selective alignment strategy that centers on efficient key token selection without requiring strong, fine-grained supervision signals. We show that Direct Preference Optimization (DPO)~\citep{rafailov2024direct} inherently learns a reward function that decouples the response-level reward values into token level~\citep{rafailov2024r}. Based on this conclusion, we propose the first DPO-based token selection method, which trains an oracle model on a moderate-scale subset of the target data distribution, aiming to parameterize an optimal token-level reward function. This token selection strategy has two key advantages: 1) \textbf{Flexibility}: the oracle modeling process is based on the original response-level preference annotations without requiring any extra supervision signals, making it directly applicable to any existing alignment datasets; 2) \textbf{Efficiency}: the cost for token selection can be easily reduced by controlling the oracle model size and the scale of its training subset, which enables cost-efficient selective alignment. We then utilize the estimated reward function to score all tokens within the large-scale target dataset, where tokens with the highest reward values in the chosen responses and tokens with the lowest reward values in the rejected responses are selected as key tokens for alignment. Finally, we design a reference model-free contrastive objective function to optimize the target policy model on the selected tokens. 

As SePO enables small oracle models to steer selective alignment for much larger policy models, we further explore it as a new paradigm for weak-to-strong generalization~\citep{burns2023weak}. Instead of leveraging weak models
to provide supervision, 1) we leverage weak oracle models to select tokens from in-distribution data for training strong policy models; 2) we propose to train oracle models on out-of-distribution data, which select key tokens to improve target policy model performance and alleviate over-optimization~\citep{gao2023scaling,rafailov2024scaling} on the weak supervision data.


In summary, our main contributions are: 

\begin{itemize}[leftmargin=*]
    \item We propose SePO, the first DPO-based selective training strategy for preference alignment, which applies to any alignment datasets with response-level supervision signals and enables cost-efficient token selection with small-scale oracle models and training data.
    \item Explorations on weak-to-strong generalization show that weak oracle models effectively supervise strong policy models with up to 16.8$\times$ more parameters. SePO also selects useful tokens from weak data, alleviating the over-optimization problem on out-of-distribution data;
    \item We examine SePO on three public evaluation benchmarks. Experiments show that SePO significantly improves performances on six policy models and outperforms competitive baseline methods by only optimizing 30\% key tokens with up to 60\% reduction in GPU training hours.
\end{itemize}

\begin{figure*}[htpb]
\centering
\includegraphics[width=14cm,height=3.5cm]{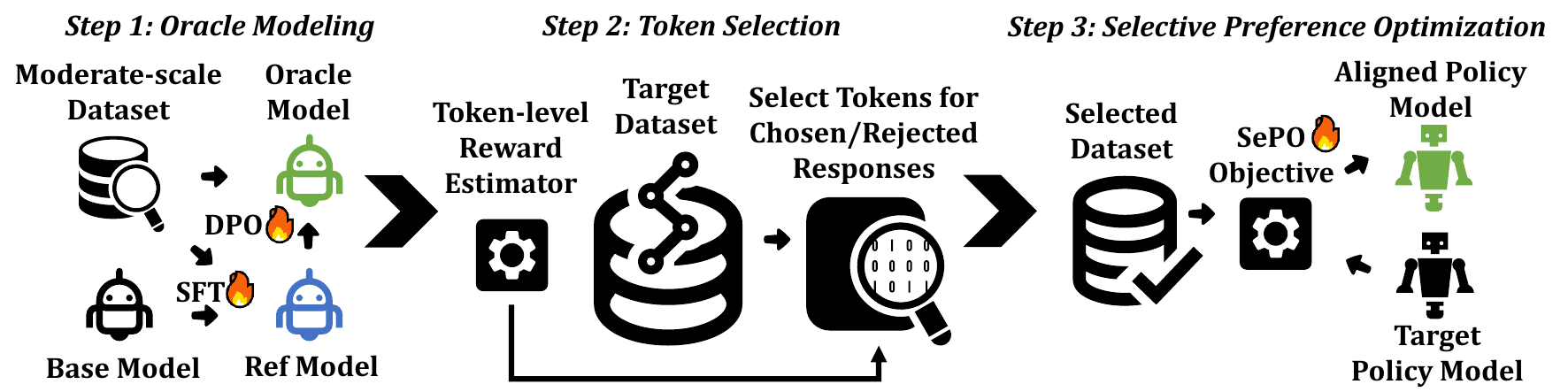}
\caption{SePO mainly consists of three steps: 1) Parameterize a token-level reward function by with a ref-oracle model pair; 2) Select key tokens within the target preference dataset; 3) Train the policy model on selected tokens.}
\label{fig:sepo_pipenline}
\end{figure*}

\section{Preliminary}
\paragraph{Alignment}
Preference alignment directly optimizes LLMs with human preferences by maximizing the reward values of model outputs, which are obtained via a response-level reward function $r(q, y)$. The reward function is defined under the Bradley-Terry~\citep{bradley1952rank} model of preferences. Specifically, for the same prompt $q$ and two completed responses $(y_1, y_2)$ under data distribution $\mathcal{D}$, the model assumes:


\begin{small}
\begin{equation}\label{eqn_bt}
    P_{\mathcal{D}}(y_1\succ y_2|q) = \frac{exp(r(q,y_1))}{exp(r(q,y_1))+exp(r(q,y_2))}
\end{equation}
\end{small}

where $P_{\mathcal{D}}(y_1\succ y_2)$ denotes the probability that $y_1$ is preferred against $y_2$.
The alignment of language models is typically cast as a KL-constrained optimization problem on the reward values. Taking the following Lagrangian, the problem is transformed as:

\begin{small}
\begin{equation}\label{eqn_align}
\begin{aligned}
&\mathop{argmax}\limits_{\pi}\mathbb{E}_{q\sim\mathcal{D}, y\sim\pi(y|q)}\left[r(q, y)\right]\\
&-\beta \mathbb{D}_{KL}\left[\pi_\phi(y|q)\Vert\pi_{ref}(y|q)\right]
\end{aligned}
\end{equation}
\end{small}


where $\pi$ denotes the aligned policy model, $\pi_{ref}$ denotes the reference policy model.


\paragraph{Direct Preference Optimization}
\citet{ziebart2008maximum} have shown that Eqn \ref{eqn_align} has a closed-form optimal solution, which enables the reward function to be re-parameterized by the optimal policy:

\begin{small}
\begin{equation}\label{eqn:closeform_dpo}
        r(q, y) = \beta \mathop{log}\frac{\pi^*(y|q)}{\pi_{ref}(y|q)}+\beta \mathop{log}Z(q)
\end{equation}
\end{small}

where $\pi^*$ denotes the optimal policy, and $Z(q)$ is the partition function. DPO~\citep{rafailov2024direct} bypasses the reward modeling stage by directly substituting this closed-form solution into Eqn. \ref{eqn_bt}, which cancels out $Z(q)$ as it un-changes with the same $q$, yielding the following DPO objective:

\begin{small}
\begin{equation}\label{eqn_dpo}
\begin{aligned}
    &\mathcal{L}_{DPO} = -\mathbb{E}_{\left(q,  y_w, y_l \right) \sim \mathcal{D}}
    \log \sigma\left(\beta u(x, y_w, y_l) \right)\\
    &u(x, y_w, y_l)=\log\frac{\pi_\theta(y_w|x)}{\pi_{ref}(y_w|x)}-\log\frac{\pi_\theta(y_l|x)}{\pi_{ref}(y_l|x)}\\
\end{aligned}
\end{equation}
\end{small}

where $y_w$ and $y_l$ denote the preferred and dis-preferred responses to the prompt $q$.

\section{Methodology}
We show DPO as inherently learning the best estimate on a token-level distribution of the response-level reward values (Sec. \ref{sec:dpo_estimate}). Based on this conclusion, we propose SePO, which optimizes the target policy model with selected key tokens (Sec. \ref{sec:SePO}). We also explore SePO as a new paradigm for weak-to-strong generalization (Sec. \ref{sec:w2s}). The pipeline is shown in Figure \ref{fig:sepo_pipenline}.

\subsection{Token-level Reward Function Estimator}\label{sec:dpo_estimate}
LLM decoding can be naturally formulated as a token-level MDP, a tuple $(\mathcal{S}, \mathcal{A}, f, r(s_t, a_t), \rho)$, where $\mathcal{S}$ and $\mathcal{A}$ denote the state and action space. $s_t\in\mathcal{S}$ deontes the current state, consisting of all prompt tokens and current generated tokens (i.e. $s_t = \{q|y_0,...,y_t\}$, $|$ denotes concatenation). $s_T$ denotes the terminal state. $a_t\in\mathcal{A}$ denotes the current action, where $\mathcal{A}$ is the token vocabulary. $f$ is the deterministic state transition function.
$\rho$ is an initial state distribution over prompts $q$.
$r(s_t, a_t)$ denotes the token-level reward function.

We begin with the following mild assumption that is widely proven appropriate~\citep{rafailov2024r,zhong2024dpo,zeng2024token}:
\begin{assumption}\label{theorem:assum}
In any token-level MDPs, the reward function $r$ can be decoupled as a linear combination of reward values modeled by another token-level reward function $\hat{r}$ along the trajectory:

\begin{small}
\begin{equation}\label{eqn:assum}
    r(q, \tau)=\sum_{t=1}^T\hat{r}(s_t, a_t)
\end{equation}
\end{small}

with $s_t$, $a_t$ along the token-level MDP trajectory $\tau=\{s_1, a_1, s_2,...,s_T\}$.
\end{assumption}
With the above assumption, the Bradley-Terry model in Eqn. \ref{eqn_bt} can be replaced into token level, where the chosen and rejected trajectories are assumed to start and end at the same state.

\begin{theorem}\label{theorem:equiv} 
With a reference model $\pi_{ref}$, fitting any reward functions $r$ that are consistent to the Bradley-Terry model with the DPO algorithm leads to an optimal estimation of another reward function $\hat{r}$ that decouples the response-level reward values into the token level, which satisfies:

\begin{small}
\begin{equation}\label{eqn:theorem}
    \hat{r}(s_t, a_t)\propto \mathop{log}\frac{\pi^*(a_t|s_t)}{\pi_{ref}(a_t|s_t)}
\end{equation}
\end{small}

where $\pi^*$ denotes the oracle model obtained via DPO on the reference model.
\end{theorem}

\begin{proof}[Proof Sketch]
This proof is heavily inspired by \citet{rafailov2024r}. Starting with the KL-regularized RL objective, the optimization process aims to maximize the expected cumulative reward. Under the maximum entropy RL setting, the optimal policy \( \pi^* \) is related to the Q-function and value function. The Bellman equation incorporates the KL term, relating the optimal Q-function \( Q^* \) to the token-level reward \( r(s_t, a_t) \).

By combining these relationships, the token-level reward can be expressed as:

\begin{small}
    \[ r(s_t, a_t) = \beta \log\frac{\pi^*(a_t|s_t)}{\pi_{ref}(a_t|s_t)} + V^*(s_t) - V^*(s_{t+1}), \]
\end{small}

where \( V^*(s_t) \) is the optimal value function.
Under Assumption \ref{theorem:assum}, summing the token-level rewards over all time steps yields the response-level reward. The term \( V^*(s_1) \) (the initial state's value) is constant with the same starting state, which does not affect preference comparisons. Therefore, the preference modeling depends only on the sum of the log-ratio terms. This shows that \( \pi^* \) inherently aligns an optimal token-level reward function:

\begin{small}
\[ \hat{r}(s_t, a_t) = \beta \log\frac{\pi^*(a_t|s_t)}{\pi_{ref}(a_t|s_t)}, \]
\end{small}

which indicates Eqn. \ref{eqn:theorem} and completes the proof. See Appendix \ref{proof_theorem1} for a detailed proof.
\end{proof}

This reward function marks the contribution of each action given the current state at the token level. In practice, the quality of the training data for DPO determines how well the calculated reward quantifies the token-level contribution.

\subsection{Selective Preference Optimization}\label{sec:SePO}
SePO is guided by the simple idea that "not all tokens are equally effective", which has been widely evaluated~\citep{lin2024rho,chen2024step,lai2024step}. We explore fully utilizing DPO to efficiently select the most useful tokens in modeling human preferences in LLM alignment. Firstly, we train an oracle model on a moderate-scale preference dataset with DPO, aiming to model a token-level reward function for the target data distribution. The reward function is then applied to large-scale data to score all the tokens. The policy model is only trained on selected tokens with highest scores in chosen responses and lowest scores in rejected responses, which are expected as key tokens in achieving alignment.

\paragraph{Oracle Modeling with DPO}
We present the following Theorem to prove the viability of random sampling on the target preference dataset to reduce the cost of oracle model training:

\begin{theorem}\label{theorem:random_dataset} 
Let $\mathcal{D}$ be the target preference dataset with $N$ samples, and $\mathcal{S}$ be a random selection of $m$ samples from $\mathcal{D}$ (m$\leq N$), which is drawn independently and uniformly. 
The reward function $r_{\mathcal{S}}$ modeled by fitting $\mathcal{S}$ with DPO is a pessimistic estimation of the target reward function $r_{\mathcal{D}}$. The result can be formalized as:

\begin{small}
\begin{equation}\label{eq:theorem2}
    \mathbb{E}_{\mathcal{S}}(r_{\mathcal{S}}(q, y))\leq r_{\mathcal{D}}(q, y)
\end{equation}
\end{small}

where $q,y$ denote any query-response pairs drown from $\mathcal{D}$. The equality holds when $m=N$.
\end{theorem}

\begin{proof}[Proof Sketch]
As the reward functions are parameterized via fitting the DPO algorithm, we replace Eqn. \ref{eqn:closeform_dpo} into Eqn. \ref{eq:theorem2} and reduce this inequality to comparing the expected optimal policy functions:

\begin{small}
\[
\mathbb{E}_{\mathcal{S}}[ \log \pi_{\mathcal{S}}^*(y|q) ] \leq \log \pi_{\mathcal{D}}^*(y|q)
\]
\end{small}

Since $\mathcal{S}$ is a uniform random sample from $\mathcal{D}$, the empirical distribution $P_{\mathcal{S}}$ is an unbiased estimator of the true distribution $P_{\mathcal{D}}$; that is, $\mathbb{E}_{\mathcal{S}}[ P_{\mathcal{S}}(X) ] = P_{\mathcal{D}}(X)$. Therefore, training on $\mathcal{S}$ yields an unbiased estimate of the optimal policy: $\mathbb{E}_{\mathcal{S}}[ \pi_{\mathcal{S}}^*(y|q) ] = \pi_{\mathcal{D}}^*(y|q)$.

Applying Jensen's inequality for the concave logarithm function, we have:

\begin{small}
\[
\mathbb{E}_{\mathcal{S}}[ \log \pi_{\mathcal{S}}^*(y|q) ] \leq \log \mathbb{E}_{\mathcal{S}}[ \pi_{\mathcal{S}}^*(y|q) ] = \log \pi_{\mathcal{D}}^*(y|q)
\]
\end{small}

showing that the expected log-optimal policy from $\mathcal{S}$ is less than or equal to that from $\mathcal{D}$ and completes the proof. 

See Appendix \ref{proof_theorem2} for a detailed proof.
\end{proof}

Theorem \ref{theorem:random_dataset} shows that training on a random-sampled subset of the target dataset with DPO can pessimistically estimate the target token-level reward function, and the estimation becomes increasingly accurate with increased sample sizes. 
With Theorem \ref{theorem:random_dataset}, we first perform SFT on a base model and the chosen responses of the moderate-scale dataset $\mathcal{S}$ to obtain the reference model $\pi_{ref}$:

\begin{small}
    \begin{equation}
        \mathcal{L}_{SFT} = -\mathbb{E}_{\left(q,y_w \right) \sim \mathcal{S}}
    \sum_i\log \pi_{ref}(y_w^i|q, y_w^{<i})
    \end{equation}
\end{small}

With the reference model, we further perform DPO on $\mathcal{S}$ with the objective function Eqn. \ref{eqn_dpo}
to obtain the oracle model $\pi_{ora}$. With Theorem \ref{theorem:equiv}, we can utilize $\pi_{ref}$ and the oracle model $\pi_{ora}$ to parameterize the optimal token-level reward function for human preferences, used for token selection.

\paragraph{Token Selection.}
We score all tokens within the target preference dataset with the estimated token-level reward function. Based on the token-level MDP and Theorem \ref{theorem:equiv}, for each prompt-response pair $(q, y)$, the score for token $y_i$ is calculated as follows:

\begin{small}
    \begin{equation}
        s(y_i)=\log \frac{\pi_{ora}(y_i|q,y_{<i})}{\pi_{ref}(y_i|q,y_{<i})}
    \end{equation}
\end{small}

For chosen responses, we utilize the following indicator function for selection:

\begin{small}
    \begin{equation}\label{eqn:indicator}
    \mathbb{I}_k^w(y_i) =
    \begin{cases}
      1, & \text{if $s(y_i)$ ranks in highest $k\%$ in $y$} \\
      0, & \text{otherwise} 
    \end{cases}
\end{equation}
\end{small}

For rejected responses, we change the condition for indicating $1$ to "if $s(y_i)$ ranks in lowest $k\%$ in $y$" and mark the indicator function as $\mathbb{I}_k^l(y_i)$. The intuition behind this action~\citep{rafailov2024r,zhong2024dpo} is that key tokens for chosen responses are likely to contribute high token-level rewards, while key tokens for rejected responses are likely with low token-level rewards, whose probabilities are significantly suppressed in $\pi_{ora}$ compared to the reference model.

\paragraph{SePO Objective.}
We design a simple contrastive preference optimization objective~\citep{meng2024simpo} on the target policy model $\pi_{t}$ with the selected tokens. Specifically, the objective function $\mathcal{L}_{SePO}$ is designed as follows:

\begin{small}
\begin{equation}\label{eq:optim}
-\mathbb{E}_{\left(q,  y_w, y_l \right) \sim \mathcal{D}}
    \log \sigma\left(\hat{u}(q, y_w, \mathbb{I}^w_{k_w})-
    \hat{u}(q, y_l, \mathbb{I}^l_{k_l}) - \lambda \right)
\end{equation}
\end{small}
where
\begin{small}
\[
\hat{u}(q, y, \mathbb{I}_k)= \frac{\gamma}{|y|\cdot k\%}\sum_{i=1}^{|y|}\mathbb{I}_k(y_i)\log\pi_{\theta}(y_i|q, y_{<i})
\]
\end{small}
$\gamma$, $\lambda$ are hyper-parameters, $k_w$, $k_l$ denote the token selection ratios for chosen/rejected responses, and $\hat{u}$ calculates the length-normalized log-likelihoods for selected tokens. This objective enables direct alignment of only crucial tokens, which avoids noise and prevents bias towards over-length responses~\citep{meng2024simpo,yuan2023rrhf}.


\subsection{SePO for Weak-to-Strong Generalization}\label{sec:w2s}
A unique advantage of SePO is that its cost can be easily reduced by controlling the base model size, using small oracle models to steer much stronger policy models. We further explore SePO as a new paradigm for weak-to-strong generalization~\citep{burns2023weak}, which aims to elicit strong student models with weak supervision signals. 
We propose the following two novel methods (More details in Appendix \ref{appn:weak2strong}):

\paragraph{Weak Oracle Modeling.}
We propose to leverage weak oracle models to select key tokens from in-distribution data. Our intuition is that weak supervisors (oracle models) only identify which tokens are most effective in alignment, rather than directly providing supervisions, which normally requires stronger capabilities than student models.

\paragraph{Weak Data Supervision.}
When only weak out-of-distribution data is available, we propose to leverage SePO to select key tokens from the weak dataset, and only the selected tokens are utilized to supervise the strong policy model. Oracle models trained on high-quality data are used to directly perform key token selection on OOD data.
We expect selective optimization to prevent over-optimization on the out-of-distribution data, while still leveraging effective supervision signals to further improve the policy model.

\input{./tables/main_results}

\section{Experiments}
This section introduces key experimental settings.

\subsection{Experimental Settings}
\paragraph{Models and Training Data.}
To approximate the optimal token-level reward function, we first obtain the reference models by training on \href{https://huggingface.co/datasets/HuggingFaceH4/ultrachat_200k}{UltraChat-200K}~\citep{ding2023enhancing} in an SFT manner. 
We train reference models on the TinyLLaMA-1.1B~\citep{zhang2024tinyllama} and Pythia-(70M, 160M, 410M, 1B, 1.4B)~\citep{biderman2023pythia} to facilitate research on the effect of oracle models with different sizes. For each reference model, we obtain the oracle models by further fine-tuning with DPO on \href{https://huggingface.co/datasets/HuggingFaceH4/ultrafeedback_binarized}{UltraFeedback}~\citep{cui2023ultrafeedback}. 
To obtain target policy models, we perform SFT on 3 foundation models: Pythia-(2.8B, 6.9B) and LLaMA3-Base-8B~\citep{dubey2024llama}, with UltraChat-200K. We also test SePO on well-aligned models LLaMA2-Chat-(7B,13B)~\citep{touvron2023llama}, and LLaMA3-Instruct-8B.

\paragraph{Baseline Methods.}
We compare SePO with state-of-the-art offline preference optimization methods: DPO~\citep{rafailov2024direct}, IPO~\citep{azar2024general}, RRHF~\citep{yuan2024rrhf} and SimPO~\citep{meng2024simpo}, and token-level alignment method TDPO~\citep{zeng2024token} and SparsePO~\citep{christopoulou2024sparsepo} (learnable sparse mask). To evaluate the SePO token selection, we further include a self-implemented SePO-rand baseline that randomly selects k\% tokens from the pair-wise data and optimizes via Eqn. \ref{eq:optim}.

\paragraph{Evaluation Benchmarks.}
We quantify the results by evaluating on three widely used instruction-following benchmarks: AlpacaEval 2.0~\citep{dubois2024length}, MT-Bench~\citep{zheng2024judging}, and Arena-Hard~\citep{li2024crowdsourceddatahighqualitybenchmarks}. \textbf{All judgments are performed by the latest GPT-4o model}.  More details in Appendix \ref{appn:exp_set}.

\subsection{Overall Performance}
Performances of SePO and other baseline methods on three benchmark datasets are presented in Table \ref{tab:main-results}. According to the results, SePO significantly improves performance over the base policy models, with an average of 9.05\% improvement in win rates on Arena-Hard. SePO also outperforms other strong preference optimization methods. On MT-Bench, SePO achieves the best average scores among other methods on five of six policy models, surpassing both state-of-the-art response-level methods such as SimPO and token-level method TDPO and SparsePO. Notably, \textbf{SePO models are only optimized on 30\% of the tokens trained on other methods}, leading to \textbf{39.94\%-62.34\% reduction in GPU training hours} compared to baseline methods. Further discussions on \textbf{how SePO saves GPU training hours} and \textbf{the influence of oracle model training on SePO efficiency} are shown in Appendix \ref{appn:gpu_hours}. These results directly strengthen the effectiveness of selective training strategies applied in preference optimization. Further comparisons with SePO-rand show that optimizing on randomly selected k\% tokens significantly damages the performance of selective training, proving the effectiveness of our DPO-based token selection strategy in filtering the crucial supervision signals from the training data.

On AlpacaEval 2.0, SePO continues to achieve superior performance over baseline methods in both win rates and LC win rates. 
Notably, SePO outperforms all other methods on length-controlled win rates, including SimPO and RRHF, which are specifically designed for length-normalized reward formulation. These results show that selective training strategies also enhance policy models in avoiding over-length responses. We believe the reason is that during token selection, the token-level reward function can assign the end-of-sentence tokens with low-ranking scores, which can be discarded during optimization if the response ends inappropriately (e.g. over-length or repetition). In contrast, though SimPO and RRHF design length-normalized 
rewards, the end-of-sentence tokens are still included and fitted for all training samples.



\subsection{Impact of Data Scale}
This section evaluates how training data scales of SePO and oracle modeling impact policy model performance. We focus on two research questions:

\paragraph{How do token selection rates influence SePO performance?}
We investigate the influences of token selection rates on SePO performances by introducing various combinations for chosen and rejected responses. The ratio for chosen/rejected responses is each selected from $\{0.1, 0.3, 0.5, 0.7, 0.9\}$ and matched pair-wise, with 25 combinations in total. The experimental results are presented in Figure \ref{fig:topk_surface}.

\begin{figure}[htpb]
\centering
\includegraphics[width=7.5cm,height=3.75cm]{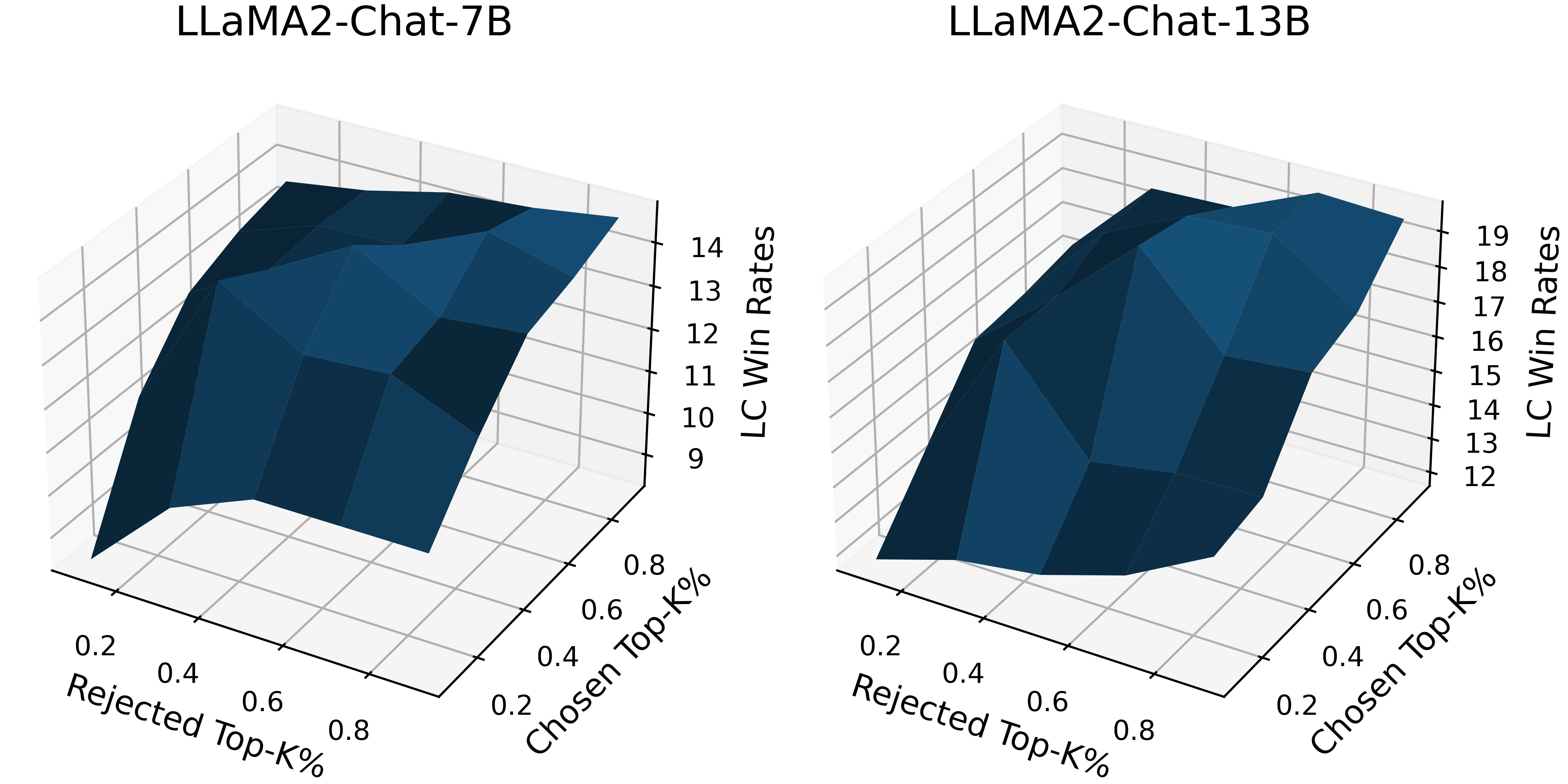}
\caption{SePO with different combinations of K\% selection ratios for chosen/rejected responses, quantified by the LC win rates on AlpacaEval 2.0.} 
\label{fig:topk_surface}
\end{figure}

According to the results, increasing selection rates from 0.1 for chosen/rejected responses rapidly improves policy model performance, but the momentum decreases as the ratio continues to grow. For example, the LC win rate of LLaMA2-Chat-7B improves from 8.37\% to 14.8\% as the ratios for chosen/rejected responses rise from 0.1 to 0.5 progressively, then stabilizes around 14.7\% with higher selection rates. These observations prove our hypothesis that not all tokens are equally effective for LLM alignment. Training only on key tokens effectively improves alignment performance, while other tokens provide limited supervision information. Training on 
Top-50\% tokens for LLaMA2-Chat-(7B, 13B) provides comparable performance to aligning on all tokens.


\begin{figure}[htpb]
\centering
\includegraphics[width=7.5cm,height=4.5cm]{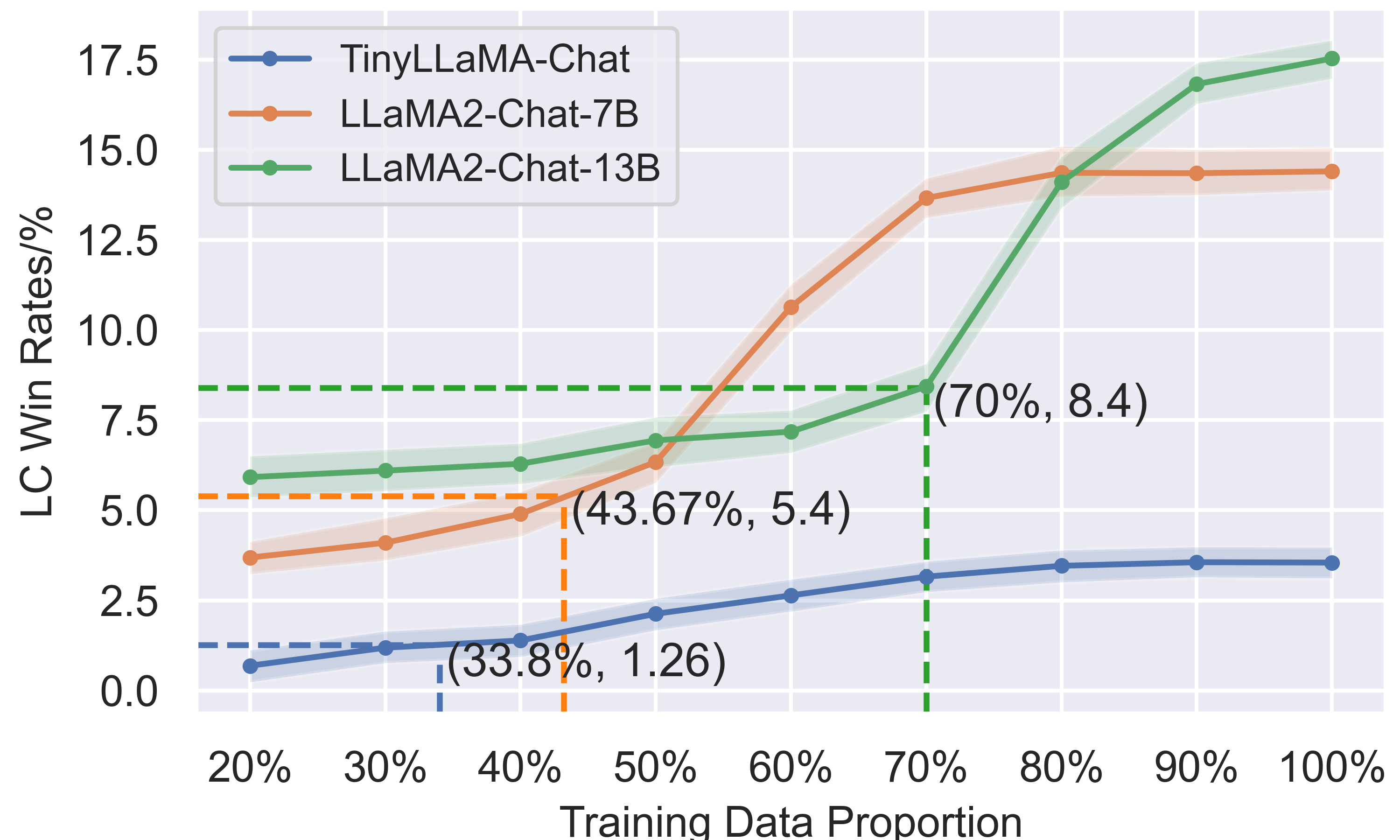}
\caption{LC win rates on AlpacaEval 2.0, supervised by oracle models trained with different data proportions. We report the average performance of 3 random runs.}
\label{fig:oracle_data}
\end{figure}

\begin{figure*}[htpb]
\centering
\includegraphics[width=14cm,height=2.8cm]{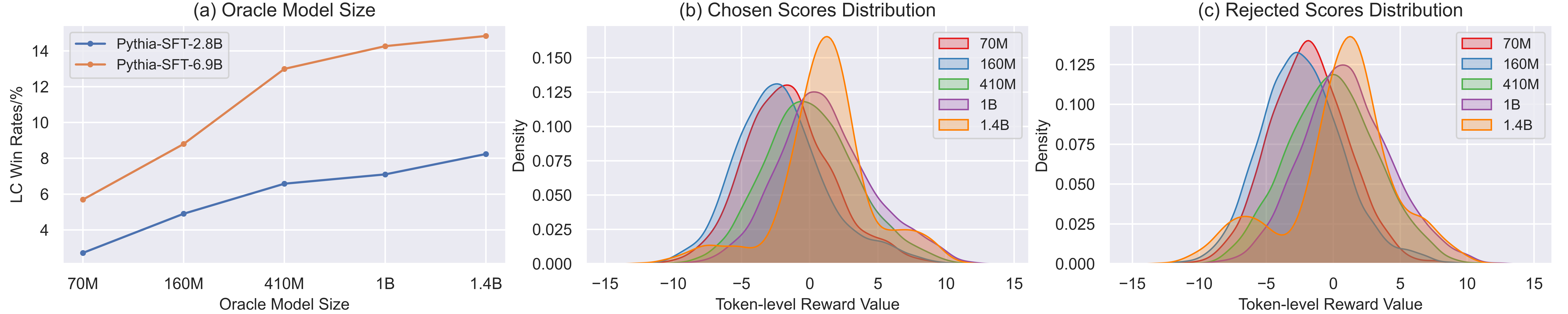}
\caption{(a) LC win rates on AlpacaEval 2.0, trained with oracle models of various sizes; (b)(c) token-level reward distributions for 5,000 chosen/rejected responses obtained from oracle models with different sizes.}
\label{fig:weak-to-strong-experiment}
\end{figure*}

\paragraph{How much data trains a good oracle model?}
In Theorem \ref{theorem:random_dataset}, we proved that training an oracle model on a moderate-scale subset is a pessimistic estimation of the target reward function. In this section, we empirically investigate the influence of the training data scale ($\frac{m}{N}$ in Theorem \ref{theorem:random_dataset}) for oracle models. Specifically, we randomly sample different proportions of data from UltraFeedback as training data for the TinyLLaMA-based oracle model. The results are shown in Figure \ref{fig:oracle_data}.

According to the results, training the oracle model on higher proportions of the target data generally leads to superior model performance. LC win rates on all policy models improve as the estimated token-level reward function becomes more accurate. Training on high data proportions also retains the majority of token selection capabilities. For example, supervising LLaMA2-Chat-7B policy model with an oracle model trained on 70\% of the data still achieves 13.66\% of LC win rates, which outperforms strong baseline methods such as SimPO and RRHF. However, the continual decrease in training proportions can significantly harm model performance. For the TinyLLaMA-Chat policy model, an oracle model trained with less than 40\% of target data leads to LC win rates of less than 1.26\%, which even underperforms the original policy model. For LLaMA2-Chat-(7B,13B), this threshold increases to 50\% and 70\%. These results indicate the importance of accurate estimation of the reward function, where false selection of key tokens degrades the capability of the policy model. These thresholds also increase with the size of policy models, showing that the high quality of key tokens becomes more important in supervising models with strong capabilities.

\subsection{Weak-to-Strong Generalization}
In this section, we empirically evaluate SePO on enhancing weak-to-strong generalization.

\paragraph{Weak Oracle Modeling.}
In Table \ref{tab:main-results}, we presented the performance of weak oracle models on guiding stronger policy models (e.g. LLaMA2-Chat-13B). The competitive results of SePO prove the viability of weak oracle modeling. 
To provide a clear landscape, we further train oracle models with Pythia-(70M, 160M, 410M, 1B, 1.4B) on the same target data and compare their performances on the Pythia-SFT-(2.8B, 6.9B) policy models. The results are shown in Figure \ref{fig:weak-to-strong-experiment} (a).

According to the results, oracle models with weak capabilities can provide effective supervision to strong policy models. For example, training with the Pythia-410M oracle model achieves 6.58\% on Pythia-SFT-2.8B and 13\% on Pythia-SFT-6.9B policy models, with up to 16.8$\times$ more parameters than the oracle model. These results outperform full optimization on the target dataset with baseline methods such as DPO and SimPO. In addition, the performance of target policy models continually improves as the oracle model size increases. For example, on Pythia-SFT-6.9B policy model, the 1.4B oracle model outperforms the 410M orcale model by 1.84\% and the 70M model by 9.15\%. These results show that oracle models with stronger capabilities can better model the token-level reward function and accurately select key tokens.

To provide an intuitive view, we present the token-level score distributions of different oracle models in Figure \ref{fig:weak-to-strong-experiment} (b)(c). For both chosen/rejected scores distribution, strong oracle models such as Pythia-(1B,1.4B) show higher densities in extreme (large and small) reward values, which facilitates separating key tokens from the other tokens. In contrast, small oracle models tend to fit a Gaussian distribution, where most tokens have similar scores. These results show that strong oracle models excel in distinguishing key tokens within the dataset,
which further proves the capability of the oracle models as crucial in accurately modeling the token-level reward function. 

\paragraph{Weak Data Supervision.}
We evaluate the weak data supervision performance of SePO by training on \href{https://huggingface.co/datasets/Anthropic/hh-rlhf}{HH-RLHF}~\citep{bai2022training}, an early-released preference dataset with relatively lower quality~\citep{yang2024metaaligner} on responses. We perform SePO with a TinyLLaMA-based oracle model and 30\% token selection rates, and comparisons with baseline methods are shown in Table \ref{tab:weak-data}.
According to the results, SePO is the only method that improves the strong LLaMA2-Chat-13B policy model with data from HH-RLHF, outperforming base performance by 1.63\% on Arena-Hard and 0.41\% on AlpacaEval 2.0. With full optimization, baseline methods such as DPO and SimPO continuously degrade model performance due to over-optimization on weak supervision data. These results prove SePO effective in leveraging useful supervision signals from weak data while avoiding over-fitting harmful patterns. These results point to SePO as a highly efficient method for continually improving strong model performance with large-scale out-of-distribution data.

\begin{table}[!hbt]
\resizebox{.47\textwidth}{!}{
\begin{tabular}{l|c|cc}
\toprule
& \multicolumn{1}{c}{\textbf{Arena-Hard}} & \multicolumn{2}{c}{\textbf{AlpacaEval 2.0}}\\
\textbf{Methods} & \textbf{Win Rate} & \textbf{LC Win Rate} & \textbf{Win Rate}\\
\midrule
Base & 12.0\% & 8.4\% & 7.7\%\\
+DPO & 10.63\% & 7.42\% & 7.18\%\\
+IPO & 9.5\% & 6.5\% & 5.98\%\\
+RRHF & 11.7\% & 7.82\% & 7.4\%\\
+SimPO & 11.39\% & 7.5\% & 7.35\%\\
+SePO & \textbf{13.63\%} & \textbf{8.81\%} & \textbf{8.4\%}\\
\bottomrule
\end{tabular}}
\caption{Weak data supervision performance.}
\label{tab:weak-data}
\end{table}

\subsection{Scaling Supervision for SePO}
To evaluate the boundary of SePO capability, we further scale the size of the oracle model to provide strong supervision beyond target model capabilities. We refer to the results in Figure \ref{fig:weak-to-strong-experiment} (a) for experimental results from the oracle model size 70M-1.4B, measured in performance on AlpacaEval 2.0. The performance shows a scaling trend as the oracle model size increases. In Table \ref{tab:strong-to-weak}, we further extend the oracle model size to 12B on supervising Pythia-(2.8B,6.9B) policy models on the
same target data.

\begin{table}[!hbt]
\resizebox{.47\textwidth}{!}{
\begin{tabular}{l|cc}
\toprule
\textbf{Oracle Model Size} & \textbf{Pythia-2.8B} & \textbf{Pythia-6.9B}\\
\midrule
\textbf{1B} & 7.1\% & 14.27\%\\
\textbf{1.4B} & 8.07\% & 15.21\%\\
\textbf{2.8B} & 9.4\% & 16.24\%\\
\textbf{6.9B} & 9.85\% & 16.8\%\\
\textbf{12B} & \textbf{10.1\%} & \textbf{17.14\%}\\
\bottomrule
\end{tabular}}
\caption{LC win rates on AlpacaEval 2.0, trained with scaling oracle model sizes from 1B to 12B.}
\label{tab:strong-to-weak}
\end{table}

According to the results, the policy models’ performance continually improves as the oracle model size increases, with about 3\% improvement on both policy models from a 1B oracle model to a 12B oracle model. These results show that the performance of SePO can further improve along with oracle models that better estimate the token selection function. However, we also observe an overall trend in decreasing momentum of the improvement as the oracle model size increases, indicating a potential information bottleneck for the selective training paradigm. These results show that the use of stronger oracle models can significantly increase SePO costs while providing a limited performance improvement.


\subsection{Related Work}
Previous works related to SePO can be divided into three parts: response-level preference optimization~\citep{ouyang2022training,stiennon2020learning,rafailov2024direct,yuan2023rrhf,meng2024simpo,ethayarajh2024kto,lu2024eliminating,azar2024general}, token-level optimization~\citep{rafailov2024r,zhong2024dpo,zeng2024token,yoon2024tlcr,chen2024low,chan2024dense,lai2024step,chen2024step}, and weak-to-strong generalization~\citep{burns2023weak,lang2024theoretical,yang2024super,charikar2024quantifying,zhou2024weak,ji2024aligner,zheng2024weak}. We provide a detailed description of related work in Appendix \ref{appn:related_work}.

\section{Conclusion}
This paper proposes SePO, an effective selective training strategy for LLM alignment. SePO estimates a token-level reward function via DPO and uses it to select key tokens from the target dataset. 
The target policy model is optimized only on the selected tokens in a contrastive manner. 
Experimental results show that SePO generally outperforms strong baseline methods by optimizing on selected tokens and shows strong promises in weak-to-strong generalization.
We also explore SePO in weak-to-strong generalization, where weak oracle models are proven to effectively supervise strong policy models and select useful supervision signals from out-of-distribution data. 

\section*{Acknowledgements}
This work is supported by the computational shared
facility at the University of Manchester and the
University of Manchester President’s Doctoral
Scholar award. This work is supported in part by Baidu Search Science. This work has been partially supported by project MIS 5154714 of the National Recovery and Resilience Plan Greece 2.0 funded by the European Union under the Next Generation EU Program.

\section*{Limitations}

\paragraph{Model Scaling}
Due to limitations in computational resources, we didn't extend our experiments to stronger policy models to provide a clear landscape of the scalability of SePO. In addition, all results in this work are obtained with a weak oracle model supervising a strong policy model. In future work, we will include more capable oracle models and policy models such as LLaMA3-Instruct-70B and the Mistral model family to further study the trends in scalability and bottlenecks of SePO. We will also examine the effect of applying a strong oracle model to weak policy models in improving their capabilities.

\paragraph{Token Ratio Selection}
Though SePO enables the model to automatically select key tokens, current token selection ratios are empirically determined. As different data sources model different data distributions, the selection ratios for key tokens can vary drastically, leaving "how to smartly decide the token selection ratios for different data and target policy models" an open problem. In future work, we will explore heuristic or training-based algorithms for the determination of token selection ratios.



\bibliography{tacl2021}

\appendix
\section{Related Work}\label{appn:related_work}
\subsection{Response-Level Preference Optimization}
With the continuous development of LLM capabilities, aligning model outputs with human values and preferences receives increasing research interests, which is commonly achieved via Reinforcement Learning from Human Feedback (RLHF)~\citep{ouyang2022training,stiennon2020learning}.
Though effective, RLHF often faces challenges like instability during training and inefficiency in requiring a separate reward model, motivating the development of direct alignment strategies.
Recent approaches have emerged to address these issues without relying on complex reward modeling. \citet{rafailov2024direct} introduce Direct Preference Optimization (DPO), a ground-breaking work that leverages a closed-form solution of the optimal policy to replace the reward values in the Bradly-Terry model, bypassing the reward modeling stage. \citet{azar2024general} provide theoretical analysis upon the framework of RLHF and DPO and propose IPO based on these insights to alleviate the over-fitting problems of DPO. \citet{yuan2023rrhf} propose Reinforcement Ranking from Human Feedback (RRHF), which aligns model outputs through a ranking loss of the response pairs, further bypassing the need for a reference model during training and minimizing the need for extensive hyperparameter tuning. Similarly, Simple Preference Optimization (SimPO)~\citep{meng2024simpo} achieve alignment via contrasting on a length-regularized implicit reward based on average log probability to improve computational efficiency and prevent the over-length preferences of DPO. SamPO~\citep{lu2024eliminating} also addresses verbosity in DPO by random-sampling the same amount of tokens from chosen and rejected responses in reward estimation. In scenarios where pair-wise data is unavailable, \citet{ethayarajh2024kto} present KTO, which integrates human biases from prospect theory into the alignment process, estimating human expectations for the responses for contrastive training.

\subsection{Token-Level Preference Optimization}
Due to the response-level supervision signals, the above alignment methods are mostly optimized on sentence bandits. This paradigm can be sub-optimal due to the sequential, auto-regressive nature of the token generation process in LLMs. This drawback has led to exploring token-level alignment methods by modeling LLM decoding as Markov Decision Processes (MDP).
Token-level DPO~\citep{zeng2024token} optimizes policy models at the token level by incorporating forward KL divergence constraints for each token, improving alignment and diversity without additional supervision signals.
Some other works introduce supervision signals at the token level. \citet{chan2024dense} use attention weights from Transformers to redistribute the response-level rewards across tokens in an unsupervised manner, aiming to stabilize the training process of RLHF.
\citet{zhong2024dpo} iteratively utilize DPO models to provide token-level rewards for each response and optimize on these token-level rewards with the PPO algorithm.
\citet{yoon2024tlcr} breaks down token-level rewards into continuous rewards by prompting powerful language models and training a discriminator.
\citet{zhou2024t} introduces token-level reward regularization via prompting LLMs.
In addition, token-level methods have been explored in related tasks such as mathematical reasoning, which require fine-grained step-wise alignment. Popular methods for obtaining token-level preferences include Monte-Carlo Tree Search~\citep{xie2024monte,chen2024step} and annotations from human/capable LLMs~\citep{lai2024step,setlur2024rl}, where they demonstrate potential in improving the precision and coherence of the target policy models. Compared to the above methods, SePO is the first method that utilizes the token-level reward function learned by DPO to perform selective preference alignment on key tokens. It proposes a token selection strategy that is more effective and efficient than previous methods, and firstly proves the viability of only optimizing on crucial supervision signals for LLM alignment.

The implicit relation between token-level rewards and DPO algorithm is first discussed by \citet{rafailov2024r}, which theoretically shows that DPO learns an inherent optimal Q-function for each action taken. Based on this intuition,
\citet{chen2024low} utilized DPO rewards to filter unimportant tokens in the rejected responses. \citet{chen2024bootstrapping} propose a self-alignment method that uses implicit DPO rewards to build new alignment data without external feedback. \citet{xia2024inverse} inverses DPO training to assign token-level reward via an superior policy. \citet{christopoulou2024sparsepo} proposes a similar selective alignment strategy via token masks with model activations or learable weights. Compared to these similar works, SePO has three main advantages: 1) the contrastive and reference model-free training paradigm reduce costs and complexity compared to reinforcement learning-based methods; 2) the oracle model can be trained once and provide selection signals for any newly-introduced data sources; 3) SePO provides a novel and effective paradigm for weak-to-strong generalization.

\subsection{Weak-to-Strong Generalization}
Weak-to-strong generalization aims to elicit the capabilities of strong student
models with weak supervision signals, which lies at the core of super alignment technologies~\citep{burns2023weak} and becomes a significant topic in the ongoing development of LLMs. This approach addresses the challenge of aligning increasingly powerful models with human values, particularly as models surpass human-level capabilities.
\citet{burns2023weak} first propose the concept and show that strong models fine-tuned on labels from weaker supervisors can outperform their weak teachers, though naive fine-tuning has limitations and may not scale well with superhuman models. \citet{lang2024theoretical} introduce a framework providing theories behind how strong models can correct weak models' errors and generalize beyond their knowledge.
\citet{yang2024super} discuss the risk of "weak-to-strong deception" where strong models exploit weak supervisors to appear aligned while misbehaving in un-monitored areas, stressing the need for more robust alignment strategies. \citet{charikar2024quantifying} quantify the performance gains of strong models over weaker ones, introducing misfit error as a key metric for optimizing this process.
Additional studies have applied the ideas of weak-to-strong generalization in tasks such as high-quality token selection~\citep{lin2024rho} and LLM alignment, including weak-to-strong search~\citep{zhou2024weak}, Aligner~\citep{ji2024aligner}, and weak-to-strong extrapolation~\citep{zheng2024weak}.

\section{Cross-Vocabulary Token Selection}
Real-world scenarios require oracle models to flexibly provide supervision signals for different policy models. Due to varied tokenization schemes, direct transfer of key tokens across different model families can be difficult.
Targeting this limitation, we implement a script that enables flexible token mappings between different tokenizers and vocabularies. Note that there are only engineering works in this implementation. The algorithm is described in Algorithm \ref{alg1}.

In the algorithm, $T$ and $c$ denotes token-level and character-level split of the target training data item. $p$ records the positions of the key tokens obtained from the oracle model. The function $split\_string$ splits the training data into characters, $get\_position$ obtains the target token's position, and $match$ checks the position of potential key characters to determine if $c$ is located in $T$. The basic logic of this transfer is to break the key tokens based on the source tokenizer into character level. For the token sequence obtained via the target tokenizer, any tokens tha contains a charater with $is\_key\_character=True$ is considered a key token in the new sequence. The algorithm breaks down the old token-level score assignments into the character level and recombine the character-level scores under the new tokenization scheme, enabling smooth transfer of key token supervision signals across different vocabularies and tokenizers. 

We use characters to align tokens because different vocabularies can tokenize words into different combinations. For example, one source tokenizer tokenizes the word “misunderstandings” into [“misunderstand”, ”ing”, “s”], while another target tokenizer could tokenize the word into [“mis”, “understanding”, “s”]. Direct word-level alignment from either source or target words can only match “s”, which leads to information loss on the scores during transformation across tokenizers. Therefore, we perform a character-level match to ensure all useful information is included. In the algorithm, both character-word matching action (line 8 and line 16 of algorithm 1) tries to match both characters and positions. For example, the “s” in “sam” will only be matched once with “sam” because we also check positions. We will modify the seudo-code in the next version of the paper to avoid these misunderstandings.

We have applied the cross-vocabulary token selection algorithm in our experiments. For example, we have adapted both TinyLLaMA-1.1B and Pythia-1B results to LLaMA3 training.

\begin{algorithm}[H]
\caption{Cross-vocab token transfer.} 
\label{alg1} 
\begin{algorithmic}[1]
\REQUIRE Target training data item $D_t\in\Sigma^{Q}$, source key token positions $P_s=\{p_s^1,...,p_s^m\}$, source tokenizer $Token_s$, target tokenizer $Token_t$. 
\ENSURE Target key token positions $P_t=\{p_t^1,...,p_t^n\}$. 
\STATE $\{T_s^1,...,T_s^M\}\gets Token_s(D_t)$
\STATE $\{T_t^1,...,T_t^N\}\gets Token_t(D_t)$
        \STATE $\{c_1,...,c_Q\} \gets split\_string(D_t)$
        \STATE $P_t \gets \varnothing$
        \FOR{$c_i$ in $\{c_1,...,c_Q\}$}
            \STATE $c_i.is\_key\_character=False$
            \FOR{$T_s^j$ in $\{T_s^1,...,T_s^M\}$}
                \IF{$match(c_i, T_s^j)$} 
                   \IF{$get\_position(T_s^j)$ in $P_s$} 
                      \STATE $c_i.is\_key\_character=True$
                   \ENDIF
                \ENDIF 
            \ENDFOR
            \IF{$c_i.is\_key\_character=True$}
            \FOR{$T_t^j$ in $\{T_t^1,...,T_t^M\}$}
                \IF{$match(c_i, T_t^j)$}
                   \STATE $P_t.update(get\_position(T_t^j))$
                \ENDIF
            \ENDFOR
            \ENDIF
        \ENDFOR
\end{algorithmic} 
\end{algorithm}

\section{Proof of Theorems}
\subsection{Theorem 1}\label{proof_theorem1}
\textit{With a reference model $\pi_{ref}$, fitting any reward functions $r$ that are consistent to the Bradley-Terry model with the DPO algorithm leads to an optimal estimation of another reward function $\hat{r}$ that decouples the response-level reward values into the token level, which satisfies:}
\begin{equation}
    \hat{r}(s_t, a_t)\propto \mathop{log}\frac{\pi^*(a_t|s_t)}{\pi_{ref}(a_t|s_t)}
\end{equation}
\textit{where $\pi^*$ denotes the oracle policy obtained by DPO.}

\begin{proof}
This proof is heavily inspired by \citet{rafailov2024r}. Common policy gradient-based RL practices~\citep{schulman2017proximal} optimize Eqn. \ref{eqn_align} in token-level MDP with an entropy-bonus $\mathcal{H}(\pi)$ and a KL-constraint with $\pi_{ref}$:
\begin{small}
\begin{equation}
    \mathop{max}_{\pi}\mathbb{E}_{a_t\sim \pi(\cdot|s_t)}\sum_{t=1}^T\left[\hat{r}(s_t,a_t)+\beta\mathop{log}\pi_{ref}(a_t|s_t)+\beta\mathcal{H}(\pi)\right]
\end{equation}
\end{small}
where $s_1\sim \rho$. Its optimal solution is given by \citet{ziebart2008maximum} under the maximum entropy RL setting:
\begin{small}
\begin{equation}\label{eqn_op}
    \pi^*(a_t|s_t) = \mathop{exp}\left((Q^*(s_t, a_t)-V^*(s_t))/\beta\right)
\end{equation}
\end{small}
where $Q^*(s_t, a_t)$ is the optimal Q-function that estimates the partial returns of $a_t$ under $s_t$, and $V^*(s_t)$ estimates the expected future returns under current state $s_t$. Under a KL-divergence regularization with the reference model, the relationship between Q-function and token-level reward values can be established as follows with the Bellman equation:
\begin{small}
\begin{equation}\label{eqn:bellman}
    Q^*(s_t, a_t) = r(s_t, a_t) + \beta \log \pi_{ref}(a_t|s_t) + V^*(s_{t+1})
\end{equation}
\end{small}
where $V^*(s_T)=0$. Combining Eqn. \ref{eqn_op} and \ref{eqn:bellman}, we have:
\begin{small}
\begin{equation}
\label{eqn:advantage}
    r(s_t, a_t)=\beta \log\frac{\pi^*(a_t|s_t)}{\pi_{ref}(a_t|s_t)}+V^*(s_{t})-V^*(s_{t+1})
\end{equation}
\end{small}
Under Assumption \ref{theorem:assum}, we substitute Eqn. \ref{eqn:advantage} into Eqn. \ref{eqn:assum}, the response-level reward is factorized as follows:
\begin{small}
\begin{equation}
\label{eqn:sum}
\begin{aligned}
    r(q, \tau)&=\sum_{t=1}^Tr(s_t, a_t)\\
    &=\mathop{\sum}_{t=1}^T\beta \log\frac{\pi^*(a_t|s_t)}{\pi_{ref}(a_t|s_t)}+V^*(s_1)
\end{aligned}
\end{equation}
\end{small}

Note that in DPO, $V^*(s_1)$ remains unchanged for each response pair as they have the same start state $s_1$. This means the preference modeling process for each response pair only depends on the first term of Eqn. \ref{eqn:sum}. Therefore, we conclude that the optimal policy $\pi^*$ learned by DPO inherently fits the response-level reward value with another token-level reward function $\hat{r}(s_t, a_t)$, which is parameterized as 
\begin{equation}\label{eq:reward_param}
    \hat{r}(s_t, a_t)=\beta\mathop{log}\frac{\pi(a_t|s_t)}{\pi_{ref}(a_t|s_t)}
\end{equation}
This indicates our results in Eqn. \ref{eqn:theorem} and completes the proof.
\end{proof}

\subsection{Theorem 2}\label{proof_theorem2}
\textit{Let $\mathcal{D}$ be the target preference dataset with $N$ samples, and $\mathcal{S}$ be a random selection of $m$ samples from $\mathcal{D}$ (m$\leq N$), which is drawn independently and uniformly. Then we have:}

\textit{The reward function $r_{\mathcal{S}}$ modeled by fitting $\mathcal{S}$ with DPO is a pessimistic estimation of the target reward function $r_{\mathcal{D}}$. The result can be formalized as:}
\begin{equation}\label{eq:exp_small}
    \mathbb{E}_{\mathcal{S}}(r_{\mathcal{S}}(q, y))\leq r_{\mathcal{D}}(q, y)
\end{equation}
\textit{where $q,y$ denote any query-response pairs drown from $\mathcal{D}$. The equality holds when $m=N$.}

\begin{proof}
As the reward functions are parameterized via fitting the DPO algorithm on the datasets, we substitute Eqn. \ref{eqn:closeform_dpo} into Eqn. \ref{eq:exp_small}. As the term $\beta\log\pi_{ref}(y|q)$ and $\beta\log Z(q)$ are unrelated to $\mathcal{S}$, they are easily canceled and we transfer the proof target into comparing the optimal policy functions:
\begin{equation}\label{eq:target}
    \mathbb{E}_{\mathcal{S}}\left[\log \pi_{\mathcal{S}}^*(y|q)\right]\leq\log \pi_{\mathcal{D}}^*(y|q)
\end{equation}

Let $\mathcal{D}=\{x_1, x_2, ..., x_N\}$, where $x_i$ represents an $(q,y,r)$ data point, and $\mathcal{S}=\{x_{i_1}, x_{i_1}, ..., x_{i_m}\}$, where $x_{i_j}$ is selected from $\mathcal{D}$. 
To show that $S$ is an unbiased estimator of the target data distribution, we calculate its empirical distribution over all possible random samples drawn from $\mathcal{D}$. The empirical distribution $P_{\mathcal{S}}(X)$ based on the sampled dataset is as follows: 
\begin{equation}
    P_{\mathcal{S}}(X)=\frac{1}{m}\sum_{j=1}^m\delta(X=x_{i_j})
\end{equation}
where $\delta$ indicates the presence of a sample $X$. Taking its expectation over all possible sampled datasets, we have:
\begin{small}
\begin{equation}\label{eq:exp_pd2}
\begin{aligned}
    \mathbb{E}_{\mathcal{S}}\left[P_{\mathcal{S}}(X)\right]&=\mathbb{E}_{\mathcal{S}}\left[\frac{1}{m}\sum_{j=1}^m\delta(X=x_{i_j})\right]\\
    &=\frac{1}{m}\sum_{j=1}^m\mathbb{E}_{\mathcal{S}}\left[\delta(X=x_{i_j})\right]
\end{aligned}
\end{equation}
\end{small}
As each $x_{i_j}$ is equally likely to be any $x_i\in\mathcal{D}$, we have
\begin{equation}\label{eq:exp_pd1}
\begin{aligned}
\mathbb{E}_{\mathcal{S}}\left[\delta(X=x_{i_j})\right]&=\frac{1}{N}\sum_{i=1}^N\delta(X=x_i)\\
&=P_{\mathcal{D}}(X)
\end{aligned}
\end{equation}
Substituting Eqn. \ref{eq:exp_pd1} into Eqn. \ref{eq:exp_pd2}, we have
\begin{equation}\label{eq:exp_pd3}
\mathbb{E}_{\mathcal{S}}\left[P_{\mathcal{S}}(X)\right]=P_{\mathcal{D}}(X)
\end{equation}
Based on the same reference model and empirical data distribution (Eqn. \ref{eq:exp_pd3}), we expect training on $S$ with DPO to obtain an unbiased estimation of the target optimal policy function:
\begin{equation}\label{eq:intermediate}
    \mathbb{E}_{\mathcal{S}}\left[ \pi_{\mathcal{S}}^*(y|q)\right]=\pi_{\mathcal{D}}^*(y|q)
\end{equation}
Because logarithm is a strictly concave function, according to Jensen's inequality, we have:
\begin{equation}\label{eq:jensen}
    \mathbb{E}_{\mathcal{S}}\left[\log \pi_{\mathcal{S}}^*(y|q)\right]\leq \log\mathbb{E}_{\mathcal{S}}\left[ \pi_{\mathcal{S}}^*(y|q)\right]
\end{equation}
Substituting Eqn. \ref{eq:intermediate} into Eqn. \ref{eq:jensen}, we prove Eqn. \ref{eq:target}, which completes the proof.
Note that when $m=N$, we have $\mathcal{S}=\mathcal{D}$, and the training process gives an unbiased estimation of target reward function $r_{\mathcal{D}}$.
\end{proof}

\section{Reward Accumulation}\label{appn:reward_accumulate}
\paragraph{Reward assignment via GPT-4}
The basic intuition of this work is that the token-level contribution of the response-level reward values is unevenly distributed, which provides opportunities for selective training on key tokens to achieve efficient alignment. We provide a direct illustration by utilizing GPT-4~\citep{achiam2023gpt} to annotate the token-level contributions of 1,000 randomly sampled query-response pairs from UltraFeedback~\citep{cui2023ultrafeedback} and QA-Feedback~\citep{wu2024fine} dataset. We tokenize each query-response pair with the LLaMA2 tokenizer and vocabulary, and include them in the prompts. For UltraFeedback, we focus on the objective of Helpfulness and use the following prompting template to obtain the scores:

\begin{center}
    \fcolorbox{black}{gray!10}{
    \parbox{.95\linewidth}{
    \textit{You are an assistant to human. You will be provided with a query and a response. For the objective of helpfulness, you will be provided with a human rating of this response ranging from 1 to 5. Consider the contribution of each token to this human rating and distribute the response-level rating to each response token. Here is an example:}\\
    \\
    \textit{Query:} What are some cool countries to visit in Asia?\\
    \textit{Response:} [“Hm"; “,"; “it"; “’s"; “difficult"; “to"; “pick"; “just"; “one"; “."
“Thailand"; “,"; “Japan"; “,"; “Vietnam"; “,"; “Indonesia"; “,"; “and"; “many"; “others"; “have"; “unique"; “history"; “and"; “culture"; “."]\\
    \textit{Human Rating:} 2 \\
    \textit{Token-level Reward:} [0.03; 0.01; 0.01; 0.01; 0.1; 0.02; 0.03; 0.05; 0.07; 0.01; 0.15; 0.01; 0.35; 0.01; 0.29; 0.01; 0.25; 0.01; 0.01; 0.07; 0.1; 0.03; 0.1; 0.12; 0.01; 0.13; 0.01]\\
    \\
    \textit{Following the format of the above example, consider and distribute the token-level reward for the following pair:}\\
    \textit{Query:} \{$\mathcal{Q}$\}\\
    \textit{Response:} \{$\mathcal{R}$\}\\
    \textit{Human Rating:} \{$\mathcal{S}$\} \\
    \textit{Token-level Reward:} 
    }
}
\end{center}
In the prompt, $\mathcal{Q}$ denotes the target query, $\mathcal{R}$ denotes the split tokens from the target response, and $\mathcal{S}$ denotes the corresponding human rating values obtained from the dataset annotations. For the QA-Feedback dataset, we concatenate the context and the question, and utilize a similar prompting strategy. As QA-Feedback only provides the relative preference for each response pair, we quantify the point-wise score for each response by counting their winning times against other responses, where each win is worth 1 point. Since the original data collected 4 responses for each query, the ratings of all responses are between 0 and 3.

With the GPT-4 assigned token-level rewards, we visualize their distributions by unifying their contributions to the response-level reward by percentages. For a response with token-level rewards:
$$
r=[r_1, r_2, ..., r_n]
$$
where $r_i$ denotes the reward value for $i$-th token, we first sort them by their values. Specifically, we sort chosen responses (or with higher human ratings) in descending order, while we sort chosen responses (or with higher human ratings) in descending order, as we expect tokens with higher values to contribute more to the chosen actions and tokens with lower values to be crucial for the rejection actions. With the sorted rewards:
$$
r_s=[r_{s1}, r_{s2}, ..., r_{sn}]
$$
we normalize their contribution at the following token percentages: $p=[10\%, 20\%, 30\%, 40\%, 50\%, 60\%, 70\%, 80\%, 90\%]$, where at each percentage the result is calculated as follows:
$$
\mathbf{P}_i = \frac{\sum_{i=1}^{p_i\cdot|r_s|}r_{si}}{\sum_{i=1}^{|r_s|}r_{si}}
$$
These outcomes are then visualized in Figure \ref{fig:reward_distribution} to support our intuitions.

\begin{figure}[htpb]
\centering
\includegraphics[width=7cm,height=4.67cm]{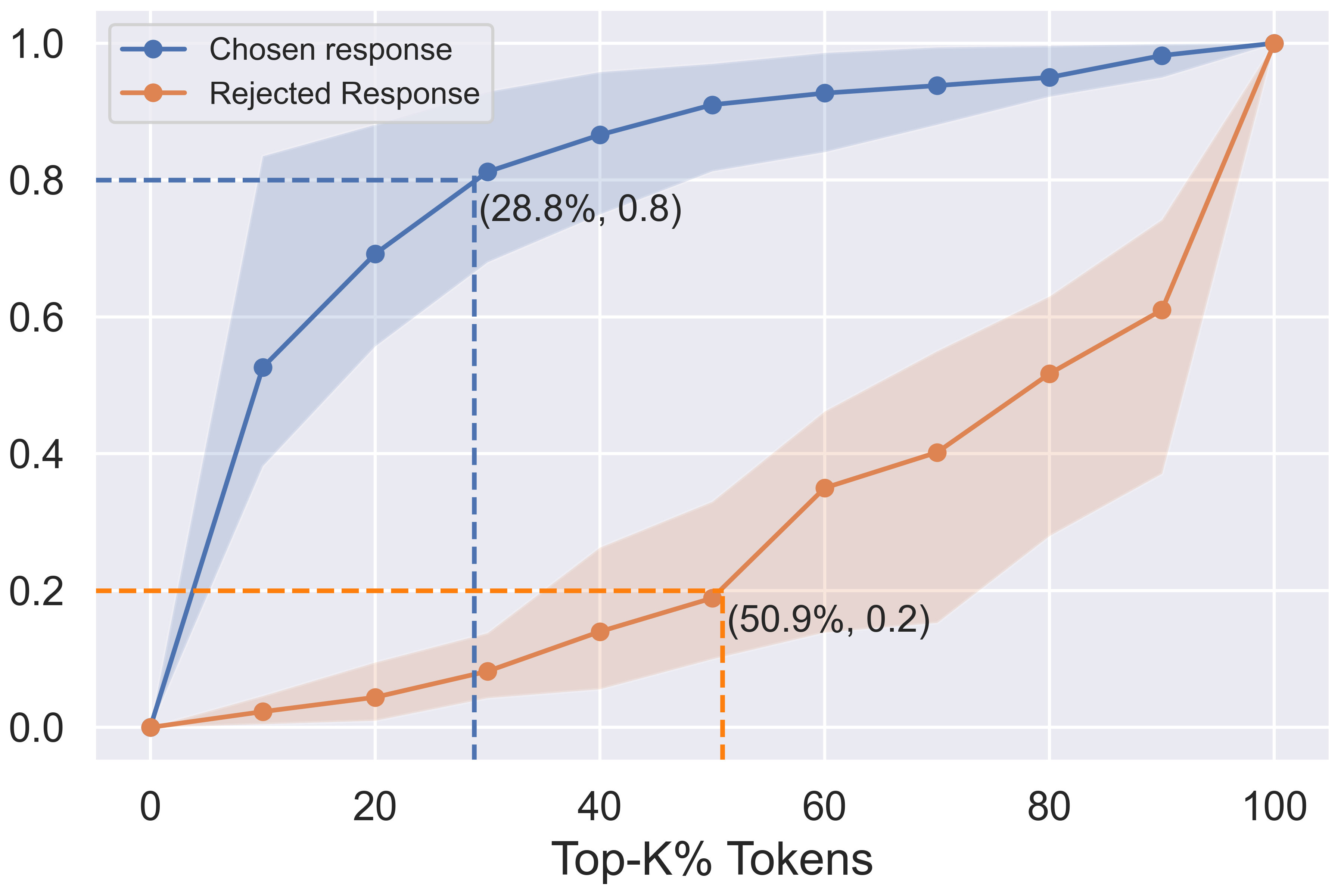}
\caption{Token-level reward accumulations on QA dataset. As tokens with high rewards are considered key tokens for chosen responses, their Top-K\% tokens are accumulated in descending order with the highest rewards. In contrast, rewards are accumulated in ascending orders for rejected responses.}
\label{fig:reward_distribution_QA}
\end{figure}

\paragraph{Results on QA Task}
To broadly validate the distribution of token-level rewards on different tasks, in Figure \ref{fig:reward_distribution_QA}, we present the token-level reward accumulations for 1,000 samples from a question answering (QA) dataset~\citep{wu2024fine}. For QA, the Top-28.8\% tokens occupy the highest 80\% rewards for chosen responses, while the lowest 50.9\% tokens only occupy 20\% rewards for rejected responses. These observations on QA dataset are similar to our conclusions on instruction following, further proving that not all tokens are equally effective in preference alignment, and optimizing on all available tokens can be noisy and inefficient.

\begin{figure*}[htpb]
\centering
\includegraphics[width=14cm,height=3.733cm]{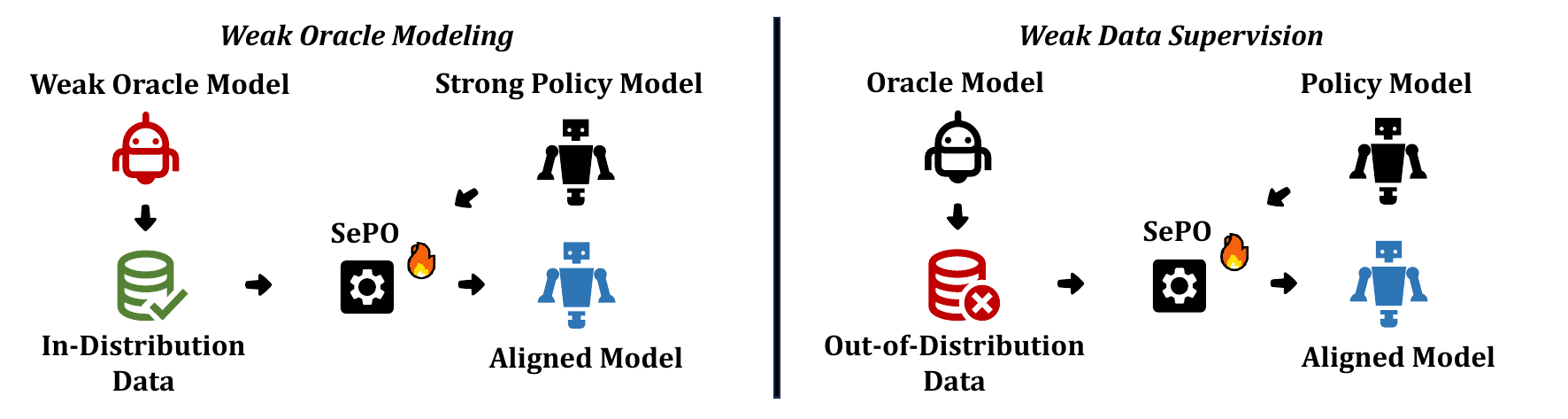}
\caption{Application of SePO on weak-to-strong generalization. Left: SePO utilizes the weak oracle model to steer the strong policy model; Right: Useful supervision signals in out-of-distribution data are selected by the oracle model to enhance alignment on the policy model.}
\label{fig:weak-to-strong}
\end{figure*}

\section{GPU Training Hours}\label{appn:gpu_hours}
We further discuss the following 2 key questions about GPU training hours for SePO:

\paragraph{How does SePO save the GPU training hours?}
Compared to other preference optimization methods such as SimPO and DPO, SePO saved GPU training hours from \textbf{reduced response sequence length}. We use the oracle model to select Top-K\% (e.g. in Table 1, K=30) tokens. During training, only the response sequence before the last Top-K\% token is retained as inputs, because the tokens after the last Top-K\% token will never be used. It is worth noting that the input of the model are not un-continuous tokens, but complete sentences with the segments after the last key token truncated. For each training key token, the model still sees a completed sentence because the model is auto-regressive from left to right. The costs for these tokens can be saved. To further strengthen this argument, we report the average sequence lengths on UltraFeedback training split before and after token selection, with LLaMA and Pythia tokenizer series:

\begin{table}[!hbt]
\resizebox{.47\textwidth}{!}{
\begin{tabular}{l|cc}
\toprule
\textbf{Tokenizer} & \textbf{Before Truncation} & \textbf{After Truncation}\\
\midrule
\textbf{LLaMA Series} & 496.07 & 383.86\\
\textbf{Pythia Series} & 444.06 & 340.18\\
\bottomrule
\end{tabular}}
\caption{Average sequence lengths on the UltraFeedback training split.}
\label{tab:data-length}
\end{table}

As shown, truncating after the last selected token leads to an average reduction of around 25\% in sequence length, significantly reducing GPU training hours.

\paragraph{Does the oracle model training cost affect the overall efficiency of SePO?}
We would like to clarify that though SePO introduce extra training cost for oracle model training, it does not affect the overall efficiency of the algorithm. Firstly, the oracle model training and token selection are \textbf{only required once} to build the selective training dataset, then the dataset is directly applicable to any policy models. For example, in Table \ref{tab:main-results}, one selected dataset is reused for multiple policy models without extra cost. Therefore, it is unfair to recount the oracle model training time for each new experiment. The efficiency of SePO also increases as the selected data is re-used. Secondly, even recounting oracle model training time for each new experiment (very unfair) results in comparable GPU training time with SimPO. The training time for all used oracle models is shown in Table \ref{tab:ora-training}. For example, training LLaMA3-Instruct-8B with Pythia-1B oracle model requires 84.36 (shown in Table \ref{tab:main-results})+32.01+8.45 (shown in Table \ref{tab:ora-training})=124.82 GPU hours, which is still comparable to SimPO (136.9 hours) and significantly outperforms DPO (201.0 hours).

\begin{table}[!hbt]
\resizebox{.47\textwidth}{!}{
\begin{tabular}{l|cc}
\toprule
\textbf{Model} & \textbf{SFT} & \textbf{DPO}\\
\midrule
TinyLLaMA-1.1B & 33.71 & 9.71\\
Pythia-70M & 3.2 & 0.93\\
Pythia-160M & 6.63 & 2.37\\
Pythia-410M & 18.2 & 7.41\\
Pythia-1B & 32.01 & 8.45\\
Pythia-1.4B & 42.48 & 11.0\\
\bottomrule
\end{tabular}}
\caption{The training time of the SFT and DPO stages for all used oracle models.}
\label{tab:ora-training}
\end{table}

\section{Weak-to-Strong Generalization}\label{appn:weak2strong}
An illustration of SePO on weak-to-strong generalization is shown in Figure \ref{fig:weak-to-strong}.
\paragraph{Weak Oracle Modeling}
The intuition behind weak oracle modeling is that weak supervisors (oracle models) only identify which tokens are most effective in enhancing the alignment performance, rather than directly providing supervision signals, which normally requires stronger capabilities than the student model.

\paragraph{Weak Data Supervision}
As the policy model becomes stronger, continual full optimization on the original data distribution can lead to over-optimization on the reward function~\citep{gao2023scaling,rafailov2024scaling}, which can seriously degrade policy model performance. Online preference optimization~\citep{xiong2024iterative,xie2024monte} alleviates over-optimization with online annotations of new in-distribution data, but can be costly for strong policy models.

In weak data supervision, instead of full optimization on the training data, we expect selective optimization on the key tokens to prevent over-optimization on the out-of-distribution data, while still leveraging effective supervision signals to further improve the policy model.

\section{Experimental Settings}\label{appn:exp_set}
\subsection{Training Details}
More details about the training process of SePO, including hardware and software we used, the training dataset information, and links to the foundation models, are listed in Table \ref{tab:training_detail}.

\subsection{Baseline Methods}
We compare the performance of SePO with the following state-of-the-art offline preference optimization methods to indicate its effectiveness. We first introduce two alignment methods that are dependent on the reference models:

\textbf{DPO}~\citep{rafailov2024direct} leverages the closed-form solution of the optimal policy model in the form of the reward function, and explicitly models their relations and substitute the reward functions in the Bradly-Terry model with the optimal policy, which enables reward model-free preference alignment with direct preference optimization. The loss function for DPO is shown in Eqn. \ref{eqn_dpo}.

\textbf{IPO}~\citep{azar2024general} sets up a general framework for preference alignment based on a generalized preference optimization objective. Based on this paradigm, it provides a variant of DPO based on identity mapping that prevents the over-fitting problem. The improved loss function is designed as follows:

\begin{small}
\begin{equation}
\begin{aligned}
    -\mathbb{E}_{\left(q,  y_w, y_l \right) \sim \mathcal{D}}
    \left(\log\frac{\pi_\theta(y_w|x)}{\pi_{ref}(y_w|x)}-\log\frac{\pi_\theta(y_l|x)}{\pi_{ref}(y_l|x)}-\frac{1}{2\lambda} \right)^2
\end{aligned}
\end{equation}
\end{small}

where $\lambda$ is a hyper-parameter.

Though the above methods achieve outstanding performance, their dependence on reference models can lead to computational inefficiency and complicated optimization process. 
We introduce another two simple yet competitive reference model-free alignment methods: 

\begin{table*}[!hbt]
\resizebox{1.\textwidth}{!}{
\begin{tabular}{lc|lc|lc}
\toprule
\textbf{Algorithm} & \textbf{Hyper-parameters} & \textbf{Algorithm} & \textbf{Hyper-parameters} & \textbf{Algorithm} & \textbf{Hyper-parameters} \\ \midrule
DPO &  $\beta\in[0.01, 0.05, 0.1]$ & IPO & $\lambda\in[0.01, 0.1, 0.5, 1.0]$ & RRHF & $\lambda\in[0.1, 0.5, 1.0, 10.0]$\\
SimPO &  $\beta\in[2.0, 2.3, 2.5]$/$\lambda\in[0.5, 1, 1.5]$ & TDPO & $\beta\in[0.01, 0.05, 0.1]$ & SePO-rand & $\gamma\in[2.0, 2.3, 2.5]$/$\lambda\in[0.5, 1, 1.5]$\\
\bottomrule
\end{tabular}}
\caption{The searched hyper-parameters for baseline models.}\label{tab_hyperpara}
\end{table*}

\textbf{RRHF}~\citep{yuan2024rrhf} directly optimize the probability of the target response pairs with a simple pairwise ranking loss, which increases the probability of preferred response and suppress the dis-preferred response. To avoid diverging too much from the original policy model, the training process is regularized with an SFT-based loss on the chosen responses. Specifically, the model is optimized via the following loss function:

\begin{small}
\begin{equation}
\begin{aligned}
    -\mathbb{E}_{\left(q,  y_w, y_l \right)} &[\mathop{max}(0, -\frac{1}{|y_w|}\mathop{log}\pi_{\theta}(y_w|x)+\\
    &\frac{1}{|y_l|}\mathop{log}\pi_{\theta}(y_l|x))-\lambda log\pi_{\theta}(y_w|x)]
\end{aligned}
\end{equation}
\end{small}

\textbf{SimPO}~\citep{meng2024simpo} focuses on the over-length bias problem of DPO that the model tends to prefer responses with redundant sequences, by introducing a length-regularized probability of the response pairs with a margin. Specifically, the SimPO objective function is formalized as follows:

\begin{small}
\begin{equation}
\begin{aligned}
    -\mathbb{E}_{\left(q,  y_w, y_l \right) \sim \mathcal{D}} &[\mathop{log}\sigma(\frac{\beta}{|y_w|}\mathop{log}\pi_{\theta}(y_w|x)-\\
    &\frac{\beta}{|y_l|}\mathop{log}\pi_{\theta}(y_l|x)-\lambda)]
\end{aligned}
\end{equation}
\end{small}

\textbf{TDPO}~\citep{zeng2024token} improves the divergence efficiency of DPO by incorporating a forward KL
divergence constraints for each token, improving both alignment and diversity without token-level supervision signals. Specifically, TDPO introduces an additional term for
fine-grained control over the KL divergence:

\begin{small}
\begin{equation}
\begin{aligned}
    -\mathbb{E}_{\left(q,  y_w, y_l \right) \sim \mathcal{D}}
    &\log \sigma(\beta \log\frac{\pi_\theta(y_w|x)}{\pi_{ref}(y_w|x)}-\beta\log\frac{\pi_\theta(y_l|x)}{\pi_{ref}(y_l|x)}\\
    &-\beta\mathbf{D}_{SeqKL}(x,y_l;\pi_{ref}||\pi_{\theta})\\
    &+\beta\mathbf{D}_{SeqKL}(x,y_w;\pi_{ref}||\pi_{\theta}))
\end{aligned}
\end{equation}
\end{small}

where $\mathbf{D}_{SeqKL}$ denotes a sequential KL-divergence.

\textbf{SePO-rand} is a self-implemented method that is used to evaluate the effectiveness of the token selection process for SePO. It bypasses the whole oracle modeling and token selection process in SePO, and randomly selects k\% tokens from the pair-wise training data. The target policy model is still optimized via Eqn. \ref{eq:optim}. To enable fair comparisons with SePO in the settings of Table \ref{tab:main-results}, we also set $k=30$ during the random selection process for SePO-rand.

We mostly follow the implementation details of SimPO on hyper-parameter search for baseline models, where the searched coefficients are listed in Table \ref{tab_hyperpara}.

\subsection{Evaluation Benchmarks}
AlpacaEval 2.0 consists of 805 queries to evaluate the models’ versatile conversational abilities. Following the standard setting, we report win rates and length-controlled (LC) win rates of evaluated models against GPT-4-turbo responses. The LC win rates are designed to reduce influences of model verbosity. MT-Bench covers eight categories with 80 queries. We report the average scores ranging from 0 to 10. Arena-Hard extends MT-Bench with 500 high-quality queries, where we follow the standard setting to report win rates against GPT-4-0314 model outputs.

\section{Additional Experimental Results}\label{appn:experiment}
\subsection{Fine-Grained Evaluation on MT-Bench}
Due to the widely reported poor separability
of MT-Bench reported by previous works~\citep{meng2024simpo,li2024crowdsourceddatahighqualitybenchmarks}, we further display fine-grained scores of model capability, which we organize 8 categories as follows: Writing, Roleplay, Extraction, Reasoning, STEM, Humanities, Math, and Coding.

\begin{table*}[!hbt]
\resizebox{1.\textwidth}{!}{
\begin{tabular}{cl|cccccccc|c}
\toprule
\textit{Policy Model} & \textbf{Methods} & \textbf{Writing} & \textbf{Roleplay} & \textbf{Reasoning} & \textbf{Math} & \textbf{Coding} & \textbf{Extraction} & \textbf{STEM} & \textbf{Humanities} & \textbf{Overall}\\
\midrule
\multirow{8}{*}{\textbf{Pythia-SFT-2.8B}} & Base & 4.25 & 4.0 & 2.45 & 1.2 & 1.8 & 2.45 & 2.7 & 4.23 & 2.8 \\
& +DPO & 4.7 & 4.4 & 2.6 & 1.55 & 2.13 & 2.55 & 3.2 & 3.95 & 3.16\\
& +IPO & 5.0 & 4.78 & 2.9 & 1.1 & 2.41 & \textbf{2.8} & 2.41 & 3.28 & 3.12\\
& +RRHF & 5.35 & 4.3 & 2.74 & 1.6 & 2.25 & 2.45 & 2.2 & 2.9 & 2.93\\
& +SimPO & 5.2 & 4.85 & 3.1 & \textbf{2.52} & 2.0 & 2.2 & 2.55 & 3.3 & 3.3\\
& +TDPO & 4.9 & 4.8 & 2.45 & 1.7 & 2.35 & 2.4 & \textbf{3.45} & 3.8 & 3.26\\
& +SePO-rand & 4.1 & 4.35 & 2.5 & 1.05 & 1.55 & 2.5 & 2.7 & 4.1 & 2.86\\
& +SePO (Ours) & \textbf{6.45} & \textbf{5.1} & \textbf{3.32} & 2.38 & \textbf{2.5} & 2.65 & 2.45 & \textbf{4.45} & \textbf{3.65}\\
\midrule
\multirow{8}{*}{\textbf{Pythia-SFT-6.9B}} & Base & 6.0 & 4.65 & 2.2 & 1.48 & 1.75 & 2.6 & 3.05 & 5.8 & 3.58 \\
& +DPO & 7.2 & 5.78 & 3.7 & 2.87 & 2.85 & 3.8 & 3.9 & 6.7 & 4.7\\
& +IPO & 7.1 & 5.45 & 3.6 & 2.5 & 2.43 & 4.35 & 3.55 & 6.65 & 4.34\\
& +RRHF & 7.4 & 4.2 & 4.2 & 2.5 & 2.4 & 3.5 & 3.4 & 6.1 & 4.31\\
& +SimPO & 8.0 & 4.8 & 4.72 & 3.13 & 2.83 & 3.15 & 3.7 & 5.7 & 4.51\\
& +TDPO & 7.55 & 5.9 & 4.3 & 2.9 & 3.0 & 3.85 & 4.2 & 6.9 & 4.78\\
& +SePO-rand & 6.3 & 4.6 & 2.05 & 1.6 & 1.85 & 2.2 & 3.45 & 5.55 & 3.45 \\
& +SePO (Ours) & \textbf{8.9} & \textbf{5.27} & \textbf{5.6} & \textbf{2.93} & \textbf{2.85} & 5.6 & \textbf{4.45} & \textbf{4.79} & \textbf{5.09}\\
\midrule
\multirow{8}{*}{\textbf{TinyLLaMA-Chat}} & Base & 4.5 & 4.6 & 2.6 & 1.45 & 2.35 & \textbf{2.95} & 3.75 & 4.1 & 3.28\\
& +DPO & 4.6 & 4.7 & 2.65 & \textbf{1.5} & 2.4 & 2.75 & 3.95 & 3.95 & 3.31\\
& +IPO & 4.9 & 4.5 & 2.65 & 1.45 & 2.4 & 2.7 & 4.25 & 4.25 & 3.38\\
& +RRHF & 4.85 & 4.75 & 2.25 & 1.3 & 2.5 & 2.75 & 4.1 & 4.75 & 3.4\\
& +SimPO & 4.9 & 4.55 & 2.1 & 1.35 & 2.25 & 2.6 & 4.6 & 5.5 & 3.28\\
& +TDPO & 4.6 & 4.9 & \textbf{2.85} & 1.4 & \textbf{2.55} & 2.6 & 4.0 & 4.45 & 3.42\\
& +SePO-rand & 4.35 & 4.85 & 2.3 & 1.35 & 2.5 & 2.9 & 3.75 & 4.1 & 3.26\\
& +SePO (Ours) & \textbf{5.55} & \textbf{5.3} & 2.55 & 1.35 & 2.25 & 2.7 & \textbf{4.5} & \textbf{5.85} & \textbf{3.78}\\
\midrule
\multirow{8}{*}{\textbf{LLaMA2-Chat-7B}} & Base & 8.2 & 6.48 & 3.65 & 1.45 & 1.95 & 4.79 & 6.98 & 8.775 & 4.48\\
& +DPO & 7.1 & 6.55 & 4.25 & 2.85 & 2.85 & 5.35 & 6.75 & 7.8 & 5.43\\
& +IPO & 7.5 & 6.75 & \textbf{4.7} & \textbf{3.55} & 2.85 & 5.2 & 6.7 & 8.0 & 5.64\\
& +RRHF & 6.85 & 6.5 & 4.1 & 3.05 & 2.8 & 5.11 & 6.6 & 7.8 & 5.35\\
& +SimPO & 7.2 & 6.7 & 4.5 & 3.5 & 2.85 & \textbf{5.68} & 6.85 & 7.8 & 5.63\\
& +TDPO & 7.3 & 6.8 & 4.25 & 3.0 & 2.95 & 5.6 & 6.6 & 7.95 & 5.55\\
& +SePO-rand & 8.0 & 6.6 & 3.8 & 1.35 & 1.7 & 4.9 & 6.9 & 8.58 & 5.23\\
& +SePO (Ours) & \textbf{8.24} & \textbf{7.83} & 4.65 & 3.05 & \textbf{3.2} & 5.4 & \textbf{8.0} & \textbf{9.8} & \textbf{6.38}\\
\midrule
\multirow{8}{*}{\textbf{LLaMA2-Chat-13B}} & Base & 6.9 & 6.85 & 4.3 & 3.15 & 3.3 & 6.3 & 7.15 & 7.65 & 5.7 \\
& +DPO & 7.28 & 6.9 & 4.81 & 4.1 & 3.77 & 6.6 & 7.48 & 8.15 & 5.84\\
& +IPO & 7.4 & 6.82 & 4.3 & 4.3 & 3.5 & 6.83 & 7.2 & 7.4 & 5.76\\
& +RRHF & 6.45 & 6.25 & 4.25 & 3.7 & 3.25 & 6.65 & 7.2 & 7.7 & 5.73\\
& +SimPO & 6.85 & 6.85 & 4.3 & 3.2 & 3.0 & 6.5 & 7.3 & 7.6 & 5.7\\
& +TDPO & \textbf{8.2} & 7.15 & 4.7 & \textbf{4.3} & 3.84 & 6.5 & 7.7 & 8.6 & 6.37\\
& +SePO-rand & 6.65 & 7.1 & 4.62 & 2.5 & 2.17 & 6.4 & 8.0 & 7.85 & 5.66\\
& +SePO (Ours) & 8.05 & \textbf{7.8} & \textbf{5.15} & 3.85 & \textbf{4.25} & \textbf{7.25} & \textbf{8.2} & \textbf{8.85} & \textbf{6.86}\\
\bottomrule
\end{tabular}}
\caption{Fine-grained performance of SePO and other baseline methods on MT-Bench. For SePO, the oracle models are based on TinyLLaMA-1.1B and Pythia-1B, trained on the full UltraFeedback dataset. The modeled reward function is then used to select the top-30\% tokens of chosen and rejected responses.}
\label{tab:fg-results}
\end{table*}

On MT-Bench, SePO outperforms all other methods on average scores. Due to the widely discussed poor separability of overall scores for MT-Bench, we look into category-based evaluations that provide fine-grained assessments. As shown, SePO achieves the best performances on 70\% of comparisons on Assistant and QA, indicating its significant improvement on subjective tasks that require high-level intention understanding and writing skills. However, SePO outperforms baseline methods in math and coding on only 40\% of the comparisons, underperforming baseline methods such as IPO and SimPO on several policy models. A possible reason is that objective tasks such as math and coding require coherent logic along the token-level MDP for response generation~\citep{xie2024monte,chen2024step,lai2024step}, while SePO is only optimized on selected tokens, which brings discontinuity in learning the logic during training. Baseline methods that optimize all tokens enable policy models to learn the full chain of reasoning and show advantages in objective scenarios.

\subsection{Hyper-parameter Selection for SePO}
As shown in Eqn. \ref{eq:optim}, the training process for SePO mainly involves two hyper-parameters: $\gamma$ controls the scaling of the rewards, $\lambda$ is controls the contrastive margin. To facilitate fair evaluations on other crucial factors such as token selection ratios and training data scale for oracle model, here we perform parameter search for the above two hyper-parameters, where we fix the token selection ratio as $k_w=k_l=0.3$ and the selected tokens from a TinyLLaMA-based oracle model trained on the full UltraFeedback dataset. We first tune $\gamma$ with $\lambda=0$ on TinyLLaMA-Chat, LLaMA2-Chat-7B, and LLaMA2-Chat-13B and select the value with the highest LC win rates on AlpacaEval 2.0. Due to the similar structure between our training objective and that of SimPO, we follow their settings and search within the following range: $\gamma\in[2.0, 2.1, 2.2, 2.3, 2.4, 2.5]$. The results are shown in Figure \ref{fig:hyper_search}(a). According to the results, we do not observe a significant alteration of model performance on all three policy models as $\gamma$ increases. For all models, the performance stabilizes after $\gamma$ increasing from 2.1. These results show that SePO performance is not sensitive to $\gamma$, a conclusion similar to that of SimPO.

\begin{figure*}[htpb]
\centering
\includegraphics[width=14cm,height=4.667cm]{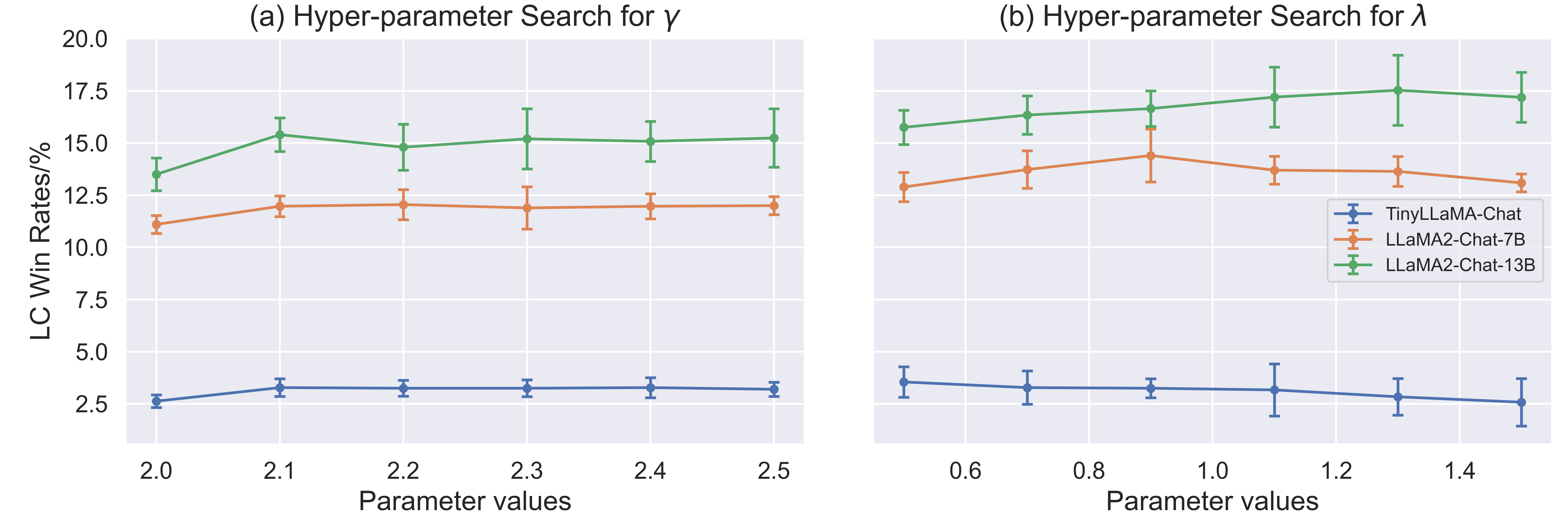}
\caption{The hyper-parameter search results on $\gamma$ and $\lambda$ for SePO. The performance is determined by LC win rates performance on the AlpacaEval 2.0 evaluation benchmark. For each hyper-parameter setting, we run the algorithm three times and show the average results.}
\label{fig:hyper_search}
\end{figure*}

Therefore, we set $\gamma=2.1$ when searching for the best $\lambda$ value. Based on the above selected value for $\gamma$, we further search $\lambda\in[0.5, 0.7, 0.9, 1.1, 1.3, 1.5]$. The results are shown in Figure \ref{fig:hyper_search}(b). According to the results, increasing $\lambda$ from 0 generally improves SePO performance on all three policy models, where results with $\lambda=0.5$ outperforms results with $\gamma=2.1$ and $\lambda=0$ on all policy models. Further increasing $\lambda$ leads to improved win rates, but with different peak performance. Models with stronger capabilities require larger margin values to reach the best performance. For example, increasing $\lambda$ from 0.5 leads to decreased LC win rates for TinyLLaMA-Chat model, while for LLaMA2-Chat-7B and LLaMA2-Chat-13B this peak value becomes 0.9 and 1.3. These results show that stronger models can generalize well to larger margin values, while weak model can over-fit to the training data when forced with larger margins.

\section{Case Studies}
We provide two cases of the key token selection process to provide a more intuitive view on how SePO works, and analyse the results of case 1 in detail. The two cases are provided in Figure \ref{fig:case_1} and \ref{fig:case_2}. We utilize the TinyLLaMA-based oracle model trained on the full UltraFeedback dataset to score the tokens. In these cases we display the tokens with highest values in the chosen response and the tokens with lowest values in the rejected response. We show the top 50\% key tokens for each response. Specifically, for chosen responses, the 10\% key tokens are marked \textcolor{blue}{blue}, the 30\% key tokens (except the 10\% key tokens) are marked \textcolor{purple}{purple}, and the 50\% key tokens (except the 30\% key tokens) are marked \textcolor{green}{green}. For rejected responses, the 10\% key tokens are marked \textcolor{red}{red}, the 30\% key tokens (except the 10\% key tokens) are marked \textcolor{orange}{orange}, and the 50\% key tokens (except the 30\% key tokens) are marked \textcolor{brown}{brown}. We expect these cases to provide intuitions into how the oracle models select key information for supervising the policy models.

According to the visualization in case 1, the top 10\% tokens tend to focus on structural features that can be generalized across instances. For example, the chosen response assigns much attention to the starting sentences: "Developing a daily habit of drawing can be challenging but with consistent practice and a few tips, it can become..." which can significantly raise the interest of the users and increase their trust on the responses. In contrast, for the rejected response, the model priorities suppressing the starting sentence of "As an AI language model I cannot personally develop habits for you.", which is negative in emotion and can decrease the users interest in continual engagement with the policy model. For the 30\% tokens, the oracle model starts to focus on the actual content of the response. In case 1, the brown parts cover the one-phrase summary of each point and improves the policy model on generating preferred suggestions for the specified query. For the 50\% tokens, the oracle model starts to focus and optimize on the details of each point. On the chosen response, the oracle model selects key statements and entities to instruct the policy model to generate factual and useful suggestions. On the rejected response, the oracle model selects less practical points such as "surround yourself with inspiration" to suppress the policy model. The oracle model also recognizes false co-references such as "everyone has \textit{their} own creative style and pace" in the rejected response.

Based on the above case studies, we conclude that the oracle model trained with DPO can rationally select key tokens for optimizing the target policy model in a explanable manner, which further proves the effectiveness of the proposed SePO algorithm.

\section{Gradient Analysis for SePO}
Similar to DPO~\citep{rafailov2024direct} and SimPO~\citep{meng2024simpo}, we calculate the gradient of SePO to provide a intuitive view of the optimization process. Different from the above works, we break down the SePO gradient $\nabla_{\theta}\mathcal{L}_{SePO}$ calculation to token level as follows:

\begin{small}
\begin{equation}
\begin{aligned}
     &-\gamma\mathbb{E}_{\left(q,  y_w, y_l \right) \sim \mathcal{D}} d_{\theta}\cdot\\
     &[\frac{1}{|y_w|\cdot k_w\%}\sum_{i=1}^{|y_w|}\mathbf{I}_k^w(s(y_i))\nabla_{\theta}\log\pi_{\theta}(y_i|q, y_{<i})-\\
    &\frac{1}{|y_l|\cdot k_l\%}\sum_{i=1}^{|y_l|}\mathbf{I}_k^l(s(y_i))\nabla_{\theta}\log\pi_{\theta}(y_i|q, y_{<i})]
\end{aligned}
\end{equation}
\end{small}

where 

\begin{small}
\begin{equation}
\begin{aligned}
    d_{\theta} &= \sigma(\frac{\gamma}{|y_l|\cdot k_l\%}\sum_{i=1}^{|y_l|}\mathbf{I}_k^l(s(y_i))\log\pi_{\theta}(y_i|q, y_{<i})\\
    &-\frac{\gamma}{|y_w|\cdot k_w\%}\sum_{i=1}^{|y_w|}\mathbf{I}_k^w(s(y_i))\log\pi_{\theta}(y_i|q, y_{<i})+\lambda)
\end{aligned}
\end{equation}
\end{small}

Firstly, similar to SimPO, the gradient weights $d_{\theta}$ of SePO is determined by likelihood of response pairs, where the weights will be higher for samples where the target policy model assigns higher likelihood to in-favored responses. The difference is that SePO only considers the incorrectly likelihoods of selected tokens that are recognized by the oracle models as key tokens. This design allows SePO to adjust weights and focuses on responses that have more misplaced key tokens, which improves the efficiency of the optimization process. 

Secondly, the updated gradients of SePO is also length-normalized, which shows an alleviation effect of bias towards redundant sequences, a feature similar to SimPO. In addition, the gradient of a token is only updated when it is selected by the indication functions $\mathbf{I}_k^l(\cdot)$ and $\mathbf{I}_k^w(\cdot)$ as key tokens. This design prevents the policy model from over-fitting to every token on the chosen/rejected responses, which allows the algorithm to update on the most effective supervision signals and ignore the irrelevant tokens that widely exist in response pairs, especially in lengthy responses. Especially, the SePO paradigm allows the model to selectively ignore optimization on end-of-sentence tokens, which further alleviates the over-optimization on lengthy responses. We believe it is also crucial for our successful application to weak-to-strong generalization, as weak data tends to include lots of noisy supervision signals, which can be filtered by the reward function to avoid weight updating during the SePO optimization process.

\newpage
\begin{figure*}[htpb]
\centering
\includegraphics[width=14cm,height=16.31cm]{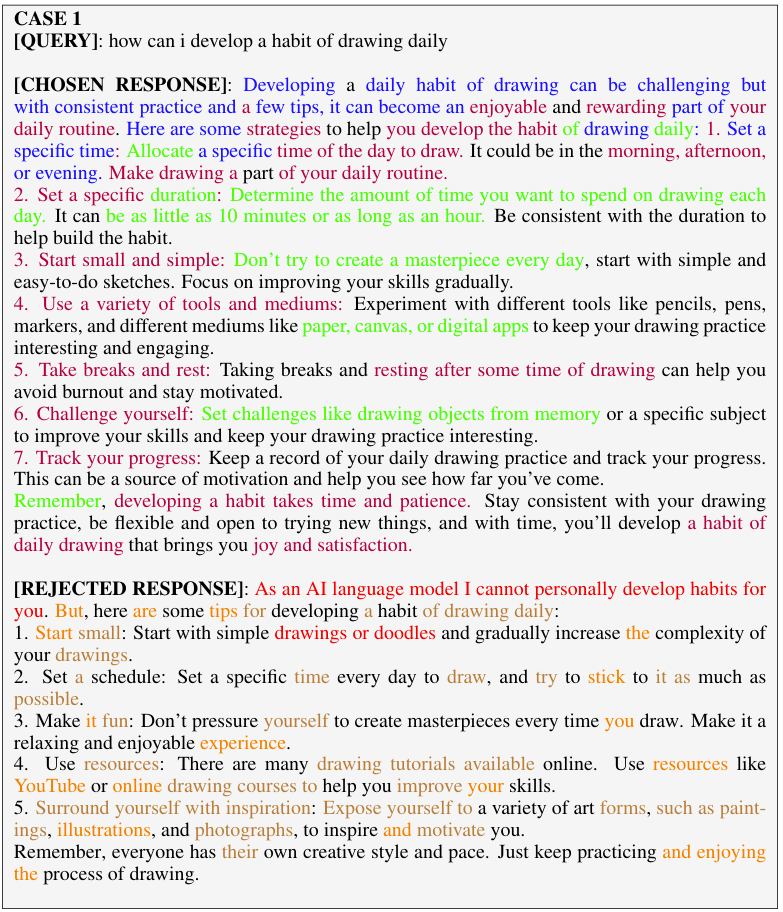}
\caption{Case 1 for visualization of the key token selection process.}
\label{fig:case_1}
\end{figure*}

\newpage
\begin{figure*}[htpb]
\centering
\includegraphics[width=14cm,height=7.41cm]{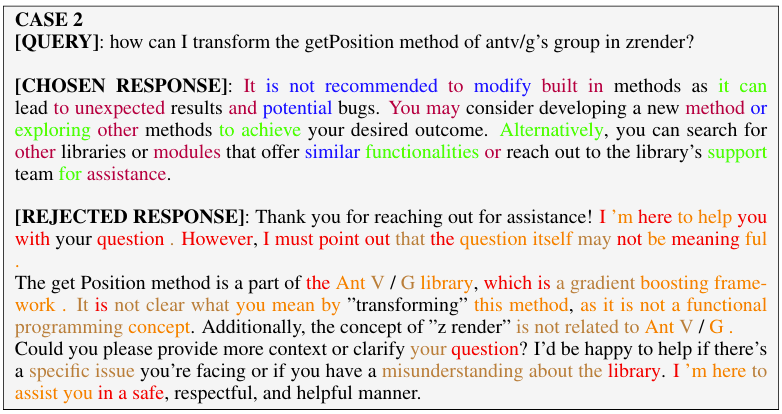}
\caption{Case 2 for visualization of the key token selection process.}
\label{fig:case_2}
\end{figure*}

\newpage
\begin{table*}[!hbt]
\resizebox{1.\textwidth}{!}{
\begin{tabular}{lc}
\toprule
\multicolumn{2}{c}{\textbf{Training Information}}\\
Base Library & \href{https://huggingface.co/}{Huggingface Transformers} \\
Fine-tuning Platform & \href{https://github.com/OpenRLHF/OpenRLHF}{OpenRLHF} \\
GPU Hardware & 4$\times$ NVIDIA Tesla A100 80GB GPUs \\
CPU Hardware & 8$\times$ Intel(R) Xeon(R) Gold 6342 CPU cores per GPU \\
Hardware Speedup & Flash Attention 2~\citep{dao2023flashattention}\\
Quantization for training & BF16\\
Supervised Fine-tuning Strategy & Full Optimization \\
Alignment Strategy & Full Optimization \\
Optimizer & Adam \\
Training Epochs &  \\
\ -SFT & 2 \\
\ -Preference Alignment & 1 \\
Batch sizes &  \\
\ -SFT & 512 \\
\ -Preference Alignment & 128 \\
Max Position Embeddings &  \\
\ -Pythia & 2048 \\
\ -TinyLLaMA & 2048 \\
\ -LLaMA2-(7B,13B) & 4096 \\
\ -LLaMA3-8B & 8192 \\
SFT Learning rate & 1e-5 \\
Preference Alignment Learning rate &  \\
\ -TinyLLaMA-Chat & 5e-7\\
\ -LLaMA2-Chat-7B & 5e-7\\
\ -LLaMA2-Chat-13B & 5e-7\\
\ -LLaMA3-Base-8B & 5e-7\\
\ -LLaMA3-Instruct-8B & 5e-7\\
\ -Pythia-2.8B & 7e-7\\
\ -Pythia-6.9B & 7e-7\\
Warm-up ratio & 0.05 \\
\midrule
\multicolumn{2}{c}{\textbf{Dataset Information}}\\
Dataset Name & \href{https://huggingface.co/datasets/HuggingFaceH4/ultrachat_200k}{\textbf{UltraChat-200K}} \\
License & MIT\\
Train/Val & 207,865/23,110\\
Data Filtering Method & Rule-based Filtering\\
Dataset Name & \href{https://huggingface.co/datasets/HuggingFaceH4/ultrafeedback_binarized}{\textbf{UltraFeedback}} \\
License & MIT\\
Train/Val & 61,135/2,000\\
Preference source & GPT-4\\
\bottomrule
\end{tabular}}
\caption{Details about SePO training and datasets.}\label{tab:training_detail}
\end{table*}
\end{document}

%% file: tables/main_results.tex
\begin{table*}[!hbt]
\resizebox{\textwidth}{!}{
\begin{tabular}{llc|c|cc|c||lc|c|cc|c}
\toprule
& && \multicolumn{1}{c}{\textbf{Arena-Hard}} & \multicolumn{2}{c}{\textbf{AlpacaEval 2.0}} & \multicolumn{1}{c}{\textbf{MT-Bench}}&&& \multicolumn{1}{c}{\textbf{Arena-Hard}} & \multicolumn{2}{c}{\textbf{AlpacaEval 2.0}} & \multicolumn{1}{c}{\textbf{MT-Bench}}\\
\textbf{Methods} & \textit{Policy Model} & \textbf{GPU Hours} & \textbf{Win Rate} & \textbf{LC Win Rate} & \textbf{Win Rate} & \textbf{GPT-4o}&\textit{Policy Model} & \textbf{GPU Hours} & \textbf{Win Rate} & \textbf{LC Win Rate} & \textbf{Win Rate} & \textbf{GPT-4o}\\
\midrule
Base & \multirow{8}{*}{\parbox{1.cm}{\textbf{Pythia-2.8B}}} & -- & 2.34\% & 3.8\% & 4.12\% & 2.8 & \multirow{8}{*}{\parbox{1.cm}{\textbf{LLaMA2-Chat-7B}}} & -- & 4.6\% & 5.4\% & 5.0\% & 4.48 \\
+DPO & & 79.45 & 5.71\% & 5.72\% & 6.1\% & 3.16 & & 182.95 & 8.5\% & 7.8\% & 6.71\% & 5.43\\
+IPO & & 75.9 & 5.6\% & 4.8\% & 4.96\% & 3.12 & & 170.3 & 8.12\% & 8.78\% & 9.4\% & 5.64\\
+RRHF & & 52.19 & 4.37\% & 4.33\% & 4.47\% & 2.93 & & 123.67 & 9.4\% & 13.35\% & 14.41\% & 5.35\\
+SimPO & & 49.31 & 5.2\% & 5.8\% & 6.0\% & 3.3 & & 119.84 & 9.59\% & 13.58\% & \textbf{15.4\%} & 5.63\\
+TDPO & & 83.1 & 6.2\% & 6.58\% & 6.8\% & 3.26 & & 204.6 & 9.23\% & 10.86\% & 10.7\% & 5.55\\
+SparsePO & & 49.3 & 4.62\% & 6.17\% & 6.35\% & 3.0 & & 85.0 & 9.76\% & 13.58\% & 12.8\% & 5.43\\
+SePO-rand & & 24.92 & 3.07\% & 4.26\% & 4.4\% & 2.86 & & 68.3 & 6.73\% & 6.38\% & 6.47\% & 5.23\\
+SePO (Ours) & & 29.84 & \textbf{6.3\%} & \textbf{7.1\%} & \textbf{7.32\%} & \textbf{3.45} & & 71.97 & \textbf{10.3\%} & \textbf{14.4\%} & 14.91\% & \textbf{6.38}\\
\midrule
Base & \multirow{8}{*}{\parbox{1.cm}{\textbf{Pythia-6.9B}}}& -- & 4.23\% & 5.0\% & 5.17\% & 3.58 & \multirow{8}{*}{\parbox{1.cm}{\textbf{LLaMA2-Chat-13B}}} & -- & 12.0\% & 8.4\% & 7.7\% & 5.7 \\
+DPO & & 190.9 & 10.2\% & 12.78\% & 13.27\% & 4.7 & & 319.67 & 13.48\% & 13.72\% & 13.37\% & 5.84\\
+IPO & & 184.74 & 8.1\% & 11.78\% & 12.6\% & 4.34 & & 336.33 & 13.95\% & 14.27\% & 14.4\% & 5.76\\
+RRHF & & 141.27 & 7.47\% & 11.42\% & 13.2\% & 4.31 & & 269.61 & 13.84\% & 15.94\% & 16.36\% & 5.73\\
+SimPO & & 145.26 & 8.0\% & 11.8\% & 12.72\% & 4.51 & & 270.35 & 14.7\% & 16.4\% & 17.02\% & 5.7\\
+TDPO && 209.93 & 10.68\% & 13.92\% & \textbf{13.7\%} & 4.78 & & 297.43 & 14.4\% & 15.0\% & 15.65\% & 6.37\\
+SparsePO & & 96.43 & 8.6\% & 12.45\% & 11.9\% & 4.52 & & 161.0 & 13.7\% & 15.6\% & 15.32\% & 5.84\\
+SePO-rand && 77.35 & 4.82\% & 5.28\% & 5.46\% & 3.45 & & 140.57 & 10.05\% & 8.16\% & 7.5\% & 5.66\\
+SePO (Ours) && 79.07 & \textbf{10.94\%} & \textbf{14.27\%} & 13.6\% & \textbf{5.09} & & 151.05 & \textbf{15.5\%} & \textbf{17.53\%} & \textbf{18.41\%} & \textbf{6.86}\\
\midrule
Base & \multirow{8}{*}{\parbox{1.cm}{\textbf{LLaMA3-Base-8B}}}& -- & 3.3\% & 6.2\% & 4.6\% & 5.4 & \multirow{8}{*}{\parbox{1.cm}{\textbf{LLaMA3-Instruct-8B}}} & -- & 20.60\% & 22.9\% & 22.6\% & 6.5 \\
+DPO && 210.0 & 15.9\% & 18.2\% & 15.5\% & 6.36 & & 201.0 & 32.6\% & 40.3\% & 37.9\% & 6.94\\
+IPO && 215.73 & 17.8\% & 14.4\% & 14.2\% & 6.7 & & 212.46 & 30.5\% & 35.6\% & 35.6\% & 6.94\\
+RRHF && 148.8 & 6.3\% & 12.1\% & 10.1\% & 5.8 & & 162.44 & 26.5\% & 34.4\% & 36.0\% & 6.7\\
+SimPO && 128.1 & 23.4\% & 22.0\% & 20.3\% & 6.76 & & 136.9 & 33.8\% & 50.0\% & \textbf{48.8\%} & \textbf{7.03}\\
+TDPO && 245.8 & 24.7\% & 23.9\% & 22.0\% & 5.94 & & 245.52 & 29.7\% & 45.0\% & 37.6\% & 6.53\\
+SparsePO & & 94.0 & 20.6\% & 21.78\% & 21.02\% & 6.4 & & 96.4 & 31.1\% & 44.34\% & 41.0\% & 6.7\\
+SePO-rand && 83.9 & 6.2\% & 10.6\% & 8.3\% & 5.6 & & 89.23 & 25.5\% & 26.74\% & 25.5\% & 6.2\\
+SePO (Ours) && 84.5 & \textbf{25.0\%} & \textbf{25.6\%} & \textbf{22.78\%} & \textbf{6.8} & & 84.36 & \textbf{34.35\%} & \textbf{50.5\%} & 47.04\% & 6.94\\
\bottomrule
\end{tabular}}
\caption{Main results. For SePO, the oracle models are fully trained TinyLLaMA-1.1B and Pythia-1B, and used to select the top-30\% tokens. "GPU Hours" records the GPU running time of each method on different policy models.}
\label{tab:main-results}
\end{table*}